\documentclass[11pt,letter]{article}
\usepackage[papersize={8.5in,11in},margin=1in]{geometry}
\usepackage{times}  
\usepackage{helvet}  
\usepackage{courier}  
\usepackage{graphicx}  
\usepackage[utf8]{inputenc} 
\usepackage[T1]{fontenc}    
\usepackage{hyperref}       
\usepackage{url}            
\usepackage{booktabs}       
\usepackage{amsfonts}       
\usepackage{nicefrac}       
\usepackage{microtype}      
\usepackage{pifont}
\usepackage{amssymb}
\usepackage{times}
\usepackage{helvet}
\usepackage{courier}
\newtheorem{theorem}{Theorem}
\newtheorem{assumption}{Assumption}
\newtheorem{proof}{Proof}
\newtheorem{remark}{Remark}
\newtheorem{corollary}{Corollary}
\newtheorem{lemma}{Lemma}
\usepackage{algorithm,algorithmicx,algpseudocode}
\usepackage{amsmath}
\usepackage{subcaption}
\usepackage{booktabs}
\usepackage{epsfig}
\usepackage{epstopdf}
\usepackage{url}
\usepackage{multirow}
\usepackage[usenames, dvipsnames]{color}
\algnewcommand{\LeftComment}[1]{   \(\triangleright\)  #1}

\renewenvironment{proof}{{ \par \noindent \textit{Proof:}}}{}
\newcommand{\cmark}{\ding{51}}%
\newcommand{\xmark}{\ding{55}}%
\newcommand*\samethanks[1][\value{footnote}]{\footnotemark[#1]}

\usepackage{authblk}

\title{Accelerated Method for Stochastic Composition Optimization with Nonsmooth Regularization}

\begin{document}

\author[1]{Zhouyuan Huo\thanks{Equal contribution}\thanks{zhouyuan.huo@pitt.edu}}
\author[1]{Bin Gu  \samethanks[1] \thanks{jsgubin@gmail.com}}
\author[2]{Ji Liu\thanks{ji.liu.uwisc@gmail.com}}
\author[1]{Heng Huang\thanks{heng.huang@pitt.edu}}
\affil[1]{Department of Electrical and Computer Engineering,  University of Pittsburgh}
\affil[2]{Department of Computer Science,  University of Rochester}
\maketitle

\begin{abstract} 
   Stochastic composition optimization draws much  attention recently and has been successful in many emerging applications of machine learning, statistical analysis, and reinforcement learning. In this paper, we focus on the composition problem with nonsmooth regularization penalty. Previous works either have slow convergence rate, or do not provide complete convergence analysis for the general problem. In this paper, we tackle these two issues by proposing a new stochastic composition optimization method for composition problem with nonsmooth regularization penalty. In our method, we apply variance reduction technique to accelerate the speed of convergence.  To the best of our knowledge, our method admits the fastest convergence rate for stochastic composition optimization: for strongly convex composition problem, our algorithm is proved to admit linear convergence; for general composition problem, our algorithm significantly improves the state-of-the-art convergence rate from $O(T^{-1/2})$ to $O((n_1+n_2)^{{2}/{3}}T^{-1})$. Finally, we apply our proposed algorithm to portfolio management and policy evaluation in reinforcement learning. Experimental results verify our theoretical analysis.
\end{abstract}

\section{Introduction}
Stochastic composition optimization draws much  attention  recently and  has been successful in addressing many emerging applications of different areas, such as reinforcement learning \cite{dai2016learning,wang2016accelerating},  statistical learning \cite{wang2017stochastic} and risk management \cite{dentcheva2016statistical}.
The authors in \cite{wang2017stochastic,wang2016accelerating} proposed  composition problem, which is the composition of two expected-value functions:
\begin{eqnarray}\label{formulation0}
	\min_{x\in \mathbb{R}^N} \ \ \underbrace{\mathbb{E}_{i}F_i(\mathbb{E}_{j}G_j(x))}_{f(x)}+h(x),
\end{eqnarray}
where $ G_j(x): \mathbb{R}^N \mapsto \mathbb{R}^M$ are inner component functions, $F_i(y): \mathbb{R}^M \mapsto \mathbb{R}$ are outer component functions. The regularization penalty $h(x)$  is a closed convex function but not necessarily smooth.  In reality, we usually solve the finite-sum scenario for composition problem (\ref{formulation0}), and it can be represented as follows:
\begin{eqnarray}\label{formulation1}
	\min_{x\in \mathbb{R}^N} H(x) = \min_{x\in \mathbb{R}^N} \underbrace{\frac{1}{n_1} \sum_{i=1}^{n_1} F_i \left (  \frac{1}{n_2} \sum_{j=1}^{n_2} G_j(x) \right )}_{f(x)} + h(x),
\end{eqnarray}
where it is defined that $F(y)  = \frac{1}{n_1} \sum\limits_{i=1}^{n_1} F_i(y) $ and $G(x) =  \frac{1}{n_2} \sum\limits_{j=1}^{n_1} G_j(x)$. Throughout this paper, we mainly focus on the case that $F_i$ and $G_j$ are smooth.  However, we do not require that $F_i$ and $G_j$ have to be convex.

Minimizing the composition of expected-value functions (\ref{formulation0}) or finite-sum functions (\ref{formulation1}) is challenging.  Classical stochastic gradient method (SGD) and its variants are well suited for minimizing traditional finite-sum functions \cite{bottou2016optimization}.  However, they are not directly applicable to the composition problem. To apply SGD, we need to compute the unbiased sampling gradient $(\nabla G_j(x))^T \nabla F_i(G(x))$ of problem (\ref{formulation1}), which is time-consuming when $G(x)$ is unknown. Evaluating $ G(x)$  requires traversing all inner component functions, which is unacceptable to compute in each iteration if $n_2$ is a large number.

\begin{table*}[th]
	\center
	\caption{The table shows the comparisons of SCGD, Accelerated SCGD, ASC-PG, Compositional-SVRG-1, Compositional-SVRG-2, com-SVR-ADMM and our VRSC-PG. For fair comparison, we consider query complexity in the convergence rate. We define that one query of Sampling Oracle ($\mathcal{SO}$) has three cases: (1) Given $x \in R^N$ and $j \in \{1,2,...,n_2\}$, $\mathcal{SO}$ returns
		$G_j(x) \in \mathbb{R}^M$; (2) Given $x \in R^N$ and $j \in \{1,2,...,n_2\}$, $\mathcal{SO}$ returns $\nabla G_j(x) \in \mathbb{R}^{M \times N}$; (3) Given $y \in \mathbb{R}^M$  and $i \in \{1,2,...,n_1\}$, $\mathcal{SO}$ returns $\nabla F_i(y) \in \mathbb{R}^M$.  $T$ denotes the total number of iterations, $\kappa$ is condition number and $0<\rho<1$.  }
	\label{table:methods}
	\setlength{\tabcolsep}{3mm}
	\begin{tabular}{c|c|c|c}
		\hline
		\textbf{Algorithm}  & \textbf{$h(x) \neq 0$} &  \textbf{Strongly Convex}&  \textbf{General Problem}  \\ \hline
		SCGD \cite{wang2017stochastic}& \xmark
		&  $O(T^{-{2}/{3}})$ & $O(T^{-1/4})$  \\
		Accelerated SCGD  \cite{wang2017stochastic}  &\xmark & $O(T^{-{4}/{5}})$ &  $O(T^{-{2}/{7}})$    \\
		Compositional-SVRG-1 \cite{lian2016finite} & \xmark   & $O\left(\rho^{\frac{T}{n_1+n_2+\kappa^4}}\right)$ & -     \\
		Compositional-SVRG-2  \cite{lian2016finite} & \xmark   & $O(\rho^{\frac{T}{n_1+n_2+\kappa^3}})$ & -     \\
		\hline \hline
		ASC-PG  \cite{wang2016accelerating}  &\cmark & $O(T^{-{4}/{5}})$  &  $O(T^{-{4}/{9}})$   \\
		ASC-PG \textit(if $G_j(x)$ are linear ) \cite{wang2016accelerating}  &\cmark &  $O(T^{-1})$ &   $O(T^{-{1}/{2}})$   \\
		com-SVR-ADMM \cite{svradmm}  & \cmark   & $O\left(\rho^{\frac{T}{n_1+n_2+\kappa^4}}\right)$  \footnotemark & -     \\ \hline
		VRSC-PG (Our) & {\cmark} &  \textcolor{red}{$O\left(\rho^{\frac{T}{n_1+n_2+\kappa^3}}\right)$}  &{\small \textcolor{red}{$O((n_1+n_2)^{2/3}T^{-1})$}}   \\
		\hline
	\end{tabular}
\end{table*}
In \cite{wang2017stochastic}, the authors considered the problem (\ref{formulation0}) with $h(x)=0$ and proposed stochastic compositional gradient descent algorithm (SCGD) which is the first stochastic method for composition problem. In their paper, they proved that the  convergence rate of SCGD for strongly convex composition problem is $O(T^{-{2}/{3}})$, and for general problem is $O(T^{-1/4})$.
They also proposed accelerated SCGD by using Nesterov smoothing technique \cite{nesterov1983method} which is proved to admit faster convergence rate. SCGD has constant query complexity per iteration, however, their convergence rate is far worse than full gradient method because of the  noise induced by sampling gradients.  Recently, variance reduction technique \cite{johnson2013accelerating} was applied to accelerate the convergence of stochastic composition optimization. \ \cite{lian2016finite} first utilized the variance reduction technique and proposed two variance reduced stochastic compositional gradient descent methods (Compositional-SVRG-1 and Compositional-SVRG-2). Both methods are proved to admit linear convergence rate.
However, the methods proposed in \cite{wang2017stochastic} and \cite{lian2016finite} are not applicable to composition problem with nonsmooth regularization penalty.

Composition problem with nonsmooth regularization was then considered in  \cite{wang2016accelerating,svradmm}. In \cite{wang2016accelerating}, the authors proposed accelerated stochastic compositional proximal gradient algorithm (ASC-PG). They proved that the optimal convergence rate of ASC-PG for strongly convex problem and general problem is $O(T^{-1})$ and $O(T^{-1/2})$ respectively. However, ASC-PG suffers from slow convergence because of the noise of the sampling gradients. \cite{svradmm} proposed com-SVR-ADMM using variance reduction.  Although com-SVR-ADMM admits linear convergence for strongly convex composition problem, it is not optimal. Besides, they did not analyze the convergence for general (nonconvex) composition problem either. We review the convergence rate of stochastic composition optimization in Table \ref{table:methods}.

In this paper,  we propose variance reduced stochastic compositional proximal gradient method (VRSC-PG) for composition problem with nonsmooth regularization penalty. Applying the variance reduction technique to composition problem is nontrivial because the optimization procedure and convergence analysis are essentially different. We investigate the convergence rate of our method: for strongly convex problem,  we prove that VRSC-PG has linear convergence rate $O\left(\rho^{\frac{T}{n_1+n_2+\kappa^3}}\right)$, which is faster than com-SVR-ADMM; For general problem, sometimes nonconvex, VRSC-PG significantly improves the state-of-the-art convergence rate of ASC-PG from $O(T^{-1/2})$ to $O((n_1+n_2)^{{2}/{3}}T^{-1})$. To the best of our knowledge, our result is the new benchmark for stochastic composition optimization. We further evaluate our method by applying it to portfolio management and reinforcement learning. Experimental results verify our theoretical analysis.

\footnotetext{ In \cite{svradmm} , their result is $O(\rho^{\frac{T}{n_1+n_2+Am}})$. We prove that to get linear convergence, it must be satisfied that $A$ and $m$ are proportional to  $\kappa^2$, which is not included in their paper. Check Remark \ref{foradmm} in supplementary material.}

\section{Preliminary}
In this section, we briefly review stochastic composition optimization and proximal stochastic variance reduced gradient.
\subsection{Stochastic Composition Optimization}
The objective function of the stochastic composition optimization is the composition of expected-value (\ref{formulation0}) or finite-sum (\ref{formulation1}) functions, which is much more complicated than traditional finite-sum problem.
The full gradient of composition problem using chain rule is $\nabla f(x) = ( \nabla G(x) )^T \nabla  F(G(x))$. Given $x$, applying the classical stochastic gradient descent method in constant queries to compute the unbiased sampling gradient $ (\nabla G_j(x))^T \nabla F_i(G(x))$ is not available, when $G(x)$ is unknown yet. In problem (\ref{formulation1}), evaluating $G(x)$ is time-consuming which requires $n_2$ queries in each iteration. Therefore, classical SGD is not applicable to  composition optimization. In \cite{wang2017stochastic}, the authors proposed the first stochastic compositional gradient descent (SCGD) for minimizing the stochastic composition problem (\ref{formulation0}) with $h(x) = 0$.   In their paper, they proposed to use an auxiliary variable $y$ to approximate $G(x)$. In each iteration $t$, we store $x_t$ and $y_t$ in memory.  SCGD are briefly described in Algorithm \ref{scgd}.

\begin{algorithm}[h]
	\caption{SCGD}
	\begin{algorithmic}[1]
		\State Initialize $x_0 \in \mathbb{R}^N$, $y_0 \in \mathbb{R}^M$;
		\For{$t=0,1,2,\dots, T-1$}
		\State  Uniformly sample $j$ from $\{1,2,...,n_2\}$  with replacement and query $G_j(x_t)$ and $\nabla_j G(x_t)$;  \Comment {\textcolor{red}{2 queries}}
		\State Update $y_{t+1}$ using:
		\begin{eqnarray}
			y_{t+1} \leftarrow (1-\beta_t) y_t + \beta_t G_j(x_t);
		\end{eqnarray}
		\State Uniformly sample  $i$ from $\{1,2,...,n_1\}$ with replacement and query $\nabla F_i(y_{t+1})$; \Comment{\textcolor{red}{1 query}}
		\State Update $x_{t+1}$ using:
		\begin{eqnarray}
			x_{t+1} \leftarrow x_t - \alpha_t (\nabla G_j(x_t))^T \nabla F_i(y_{t+1}) ;
		\end{eqnarray}
		\EndFor
	\end{algorithmic}
	\label{scgd}
\end{algorithm}

In the algorithm, $\alpha_t$ and $\beta_t$ are learning rate. Both of them are decreasing to guarantee convergence because of the noise induced by sampling gradients. In their paper, they supposed that $x \in \mathcal{X}$. In each iteration, $x$  is projected to $\mathcal{X}$ after step $4$. Furthermore, the authors proposed Accelerated SCGD by applying Nesterov smoothing \cite{nesterov1983method}, which is proved to converge faster than basic SCGD.

\subsection{Proximal Stochastic Variance Reduced Gradient}
Stochastic variance reduced gradient (SVRG) \cite{johnson2013accelerating} was proposed to minimize finite-sum functions:
\begin{eqnarray}
	\min\limits_{x \in \mathbb{R}^N}	\frac{1}{n_1} \sum\limits_{i=1}^{n_1} f_i(x),
	\label{finite}
\end{eqnarray}
where component functions $f_i(x) : \mathbb{R}^N \rightarrow  \mathbb{R}$. In large-scale optimization, SGD and its variants use unbiased sampling gradient $\nabla f_i(x)$ as the approximation of the full gradient, which only requires one query in each iteration. However, the variance induced by sampling gradients forces us to decease learning rate to make the algorithm converge. Suppose $x^*$ is the optimal solution to problem (\ref{finite}), full  gradient $ \frac{1}{n_1}\sum\limits_{i=1}^{n_1} \nabla f_i(x^*) = 0$, while sampling gradient $\nabla f_i(x^*) \neq 0$.  We should decease learning rate,  otherwise the convergence of the objective function value can not be guaranteed.  However, the decreasing learning rate makes SGD converge very slow at the same time. For example, if problem (\ref{finite}) is strongly convex, gradient descent method (GD) converges with linear rate, while SGD converges with a learning rate at $O(T^{-1})$. Reducing the variance is one of the most important ways to accelerate SGD, and it has been widely applied to large-scale optimization \cite{bottou2016optimization,gu2016zeroth,huo2017asynchronous}. In \cite{johnson2013accelerating}, the authors proposed to use the aggregation of old sampling gradients to regulate the current sampling gradient. In \cite{xiao2014proximal}, the authors considered the nonsmooth regularization penalty $h(x) \neq 0$ and proposed proximal stochastic variance reduced gradient (Proximal SVRG).    Proximal SVRG is briefly described in Algorithm \ref{svrg}. In their paper, they used $v_t$ as the approximation of full gradient, where $\mathbb{E} v_t = 0$. It was also proved that the variance of $v_t$ converges to zero: $\lim\limits_{t \rightarrow \infty} \mathbb{E} \|v_t - \frac{1}{n_1} \sum\limits_{i=1}^{n_1} \nabla f_i({x_t}) \|_2^2 \rightarrow 0$.  Therefore, we can keep learning rate $\eta$ constant in the procedure. In step 7, $\text{Prox}_{\eta h(.)}(x)$ denotes proximal operator. With the definition of proximal mapping, we have:
\begin{eqnarray}
	\text{Prox}_{\eta h(.)} (x) =  \arg \min\limits_{x'} ( h(x') + \frac{1}{\eta} \|x' - x\|^2),
\end{eqnarray}
Convergence analysis and experimental results confirmed that Proximal SVRG admits linear convergence in expectation for strongly convex optimization.  In \cite{reddi2016proximal},  the authors proved that Proximal SVRG has sublinear convergence rate of $O(n_1^{2/3}T^{-1})$ when $f_i(x)$ is nonconvex.

\begin{algorithm}[h]
	\renewcommand{\algorithmicrequire}{\textbf{Input:}}
	\renewcommand{\algorithmicensure}{\textbf{Output:}}
	\caption{Proximal SVRG}
	\begin{algorithmic}[1]
		\State Initialize $\tilde{x}^0 \in \mathbb{R}^N$;
		\For{$s=0,1,2,\dots S-1$}
		\State $x_0^{s+1} \leftarrow \tilde{x}^s $;
		\State $f' \leftarrow \frac{1}{n_1} \sum\limits_{i=1}^{n_1} \nabla f_i(\tilde{x}^s)  $; \Comment {\textcolor{red}{$n_1$ queries }}
		\For{$t=0,1,2,\dots, m-1$}
		\State  Uniformly sample $i$ from $\{1,2,...,n_1\}$ with replacement and query $\nabla f_i(x_{t}^{s+1})$ and $\nabla f_i(\tilde{x}^{s}) $;\Comment {\textcolor{red}{$2$ queries }}
		\State Update $v_t^{s+1}$ using:
		\begin{eqnarray}
			v_t^{s+1} \leftarrow \nabla f_i(x_{t}^{s+1}) - \nabla f_i(\tilde{x}^{s}) + f';
		\end{eqnarray}
		\State Update model $x_{t+1}^{s+1}$ using:
		\begin{eqnarray}
			x_{t+1}^{s+1} \leftarrow \text{Prox}_{\eta h(.)}( x_t^{s+1} - \eta v_t^{s+1});
		\end{eqnarray}
		\EndFor
		\State $\tilde x^{s+1} \leftarrow x_{m}^{s+1}$;
		\EndFor
	\end{algorithmic}
	\label{svrg}
\end{algorithm}

\section{Variance Reduced Stochastic Compositional Proximal Gradient}
In this section, we propose variance reduced stochastic compositional proximal gradient method (VRSC-PG) for solving the finite-sum composition problem with nonsmooth regularization penalty (\ref{formulation1}).

The description of VRSC-PG  is presented in  Algorithm \ref{algorithmSVRG2}. Similar to the framework of Proximal SVRG \cite{xiao2014proximal},  our VRSC-PG also has two-layer loops. At the beginning of the outer loop $s$, we keep a snapshot of the current model $\tilde{x}^{s}$ in memory and  compute the full  gradient:
\begin{eqnarray}
	\nabla f(\tilde x^s) =  \frac{1}{n_2} \sum_{j=1}^{n_2} (\nabla G_j(\tilde x^s) )^T {\frac{1}{n_1} \sum_{i=1}^{n_1} \nabla F_i \left ( {G}^s \right )},
	\label{alg_eq1}
\end{eqnarray}
where ${G}^s =  \frac{1}{n_2} \sum_{j=1}^{n_2} G_j(\tilde{x}^s)$  denotes the value of the inner functions and $ \nabla  G(\tilde{x}^{s})=\frac{1}{n_2} \sum_{j=1}^{n_2}\nabla G_j(\tilde{x}^s)$ denotes the gradient of inner functions.  Computing the full gradient of $f(x)$ in problem (\ref{formulation1}) requires $(n_1 + 2n_2)$ queries.

To make the number of queries in each inner iteration irrelevant to $n_2$, we need to keep $\widehat{G}_t^{s+1}$ and $\nabla \widehat{G}_t^{s+1}$ in memory to work as the estimates  of $G(x_t^{s+1})$ and $\nabla G(x_t^{s+1})$ respectively. In our algorithm,  we query $G_{A_t}(x^{s+1}_t)$ and $G_{A_t}(\tilde{x}^s)$, then $\widehat{G}_t^{s+1}$ is evaluated as follows:
\begin{eqnarray}
	\widehat{G}_t^{s+1} = {G}^s -  \frac{1}{A}\sum_{1 \leq j \leq A} \left (G_{A_t[j]}(\tilde x^{s}) - G_{A_t[j]}(x^{s+1}_t) \right ),
	\label{alg3_2}
\end{eqnarray}
where $A_t[{j}]$ denotes element $j$ in the set $A_t$ and $|A_t|=A$. The elements of $A_t$ are uniformly sampled from $\{1,2,...,n_2\}$ with replacement.  In (\ref{alg3_2}), we reduce the variance of $G_{A_t}(x_t^{s+1})$ by using $G^s$ and $G_{A_t}(\tilde{x}^s)$ .
Similarly, we sample $B_t$ with size $B$ from $\{1,2, ...,n_2\}$ uniformly with replacement, and query $\nabla G_{B_t} (x_t^{s+1})$ and $\nabla G_{B_t} (\tilde{x}^s) $. The estimation of $\nabla G(x_t^{s+1})$ is evaluated as follows:
{\small
	\begin{eqnarray}
		\nabla \widehat{G}_t^{s+1} = \nabla  G(\tilde{x}^{s}) -  \frac{1}{B}\sum_{1 \leq j \leq B} \left (\nabla G_{B_t[j]}(\tilde x^{s}) - \nabla G_{B_t[j]}(x^{s+1}_t) \right )\label{alg3_3}
	\end{eqnarray}
}
where $B_t[{j}]$ denotes element $j$ in the set $B_t$ and $|B_t|=B$.  It is important to note that $A_t$ and $B_t$ are independent. Computing $	 \widehat{G}_t^{s+1} $ and $\nabla \widehat{G}_t^{s+1} $ requires $(2A+2B)$ queries in each inner iteration.

Now, we are able to compute the estimate of $\nabla f(x_t^{s+1})$ in inner iteration $t$ as follows:
\begin{eqnarray}\label{alg3_1}
	{v}_{t}^{s+1} &=& \frac{1}{b_1} \sum\limits_{i_t \in I_t}  \biggl(  \left ( \nabla \widehat{G}_t^{s+1} \right )^T \nabla  F_{i_t}( \widehat{G}_t^{s+1}) \nonumber \\
	&&  - \left ( \nabla  G(\tilde x^{s}) \right )^T \nabla  F_{i_t}({G}^s)  \biggr)+ \nabla  f(\tilde{x}^{s}),
\end{eqnarray}
where $I_t$ is a set of indexes uniformly sampled from $\{1,2,...,n_1 \}$ and $|I_t|=b_1$. As per (\ref{alg3_1}), we need to query $\nabla F_{I_t} (\hat G_t^{s+1})$ and $\nabla F_{I_t} (G^s)$, and it requires $2b_1$ queries. Finally, we update the model with proximal operator:
\begin{eqnarray}
	x^{s+1}_{t+1} =  \textrm{Prox}_{ \eta h(\cdot) }\left ( x^{s+1}_t - {\eta}
	{v}^{s+1}_{t}   \right ),
\end{eqnarray}
where $\eta$ is the learning rate.


\begin{algorithm}
	\renewcommand{\algorithmicrequire}{\textbf{Input:}}
	\renewcommand{\algorithmicensure}{\textbf{Output:}}
	\caption{VRSC-PG}
	\begin{algorithmic}[1]
		\Require The total number of iterations in the inner loop $m$, the total number of iterations in the outer loop $S$,  the size of the mini-batch sets $A$,$B$ and $b_1$, learning rate $\eta$.
		\State  Initialize  $ \tilde{x}^0 \in \mathbb{R}^N$;
		\For{$s=0,1,2,\cdots,S-1$}
		\State $x_0^{s+1} \leftarrow \tilde{x}^s$;
		\State ${G}^s \leftarrow  \frac{1}{n_2} \sum_{j=1}^{n_2} G_i(\tilde{x}^s)$; \Comment {\textcolor{red}{$n_2$ queries}}
		\State $ \nabla  G(\tilde{x}^{s}) \leftarrow \frac{1}{n_2} \sum_{j=1}^{n_2}\nabla G_j(\tilde{x}^s)$;\Comment {\textcolor{red}{$n_2$ queries}}
		\State Compute the full gradient $\nabla f(\tilde x^s)$ using (\ref{alg_eq1}) ; \Comment {\textcolor{red}{$n_1$ queries}}
		\For{$t=0,1,2,\cdots,m-1$}
		\State  Uniformly sample ${A_t}$ from $\{1,2, ...,n_2\}$ with replacement and $|A_t| = A$ ;
		\State Update $\widehat{G}_t^{s+1} $ using (\ref{alg3_2}) ; \Comment {\textcolor{red}{$2A$ queries}}
		\State  Uniformly sample $B_t$ from $\{1,2, ...,n_2\}$ with replacement and $|B_t|= B$;
		\State Update  $\nabla \widehat{G}_t^{s+1} $ using (\ref{alg3_3}); \Comment {\textcolor{red}{$2B$ queries}}
		\State Uniformly sample $I_t$  from $\{1,2, ...,n_1\}$ with replacement;
		
		\State Compute $ {v}_{t}^{s+1} $ using (\ref{alg3_1}):
		\Comment {\textcolor{red}{$2b_1$ queries}}
		\State Update  model $x^{s+1}_{t+1}$ using:
		\begin{eqnarray}
			x^{s+1}_{t+1} \leftarrow \textrm{Prox}_{ \eta h(\cdot) }\left ( x^{s+1}_t - {\eta}
			{v}^{s+1}_{t}   \right )
		\end{eqnarray}
		\EndFor
		\State $\tilde{x}^{s+1}\leftarrow x_{m}^{s+1}$;
		\EndFor
	\end{algorithmic}
	\label{algorithmSVRG2}
\end{algorithm}

\section{Convergence Analysis}

In this section, we prove that (1) VRSC-PG admits linear convergence rate for the strongly convex problem; (2) VRSC-PG admits sublinear convergence rate $O((n_1+n_2)^{{2}/{3}}T^{-1})$ for the general problem.  To the best of our knowledge, both of them are the best results so far.  Following  are the assumptions commonly used for stochastic composition optimization \cite{wang2017stochastic,wang2016accelerating,lian2016finite}.

\noindent \textbf{Strongly  convex:} \ \ To analyze the convergence of VRSC-PG for the strongly convex composition problem, we assume that the function $ f$  is $\mu$-strongly convex.
\begin{assumption}\label{assumption2} The  function $f(x)$ is $\mu$-strongly convex. Therefore  $\forall x$ and $\forall y$, we have:
	\begin{eqnarray}\label{ass1_1}
		\| \nabla f(x) - \nabla f(y) \| &\geq &\mu  \|x - y \|.
	\end{eqnarray}
	Equivalently, $\mu$-strongly convexity  can  also be written as follows:
	\begin{eqnarray} \label{strongly_convex2}
		f(x)& \geq& f(y) + \langle \nabla f(y) , x-y \rangle + \frac{\mu}{2}  \left  \| x-y \right \|^2.
	\end{eqnarray}
\end{assumption}

\noindent \textbf{Lipschitz Gradient:} \ \  We assume that there exist Lipschitz constants $L_F$, $L_G$ and $L_f$ for $\nabla F_i(x)$, $\nabla G_j(x)$ and $\nabla f(x)$ respectively.
\begin{assumption}
	\label{definition1}
	There exist constants $L_F$, $L_G$ and $L_f$ for $\nabla F_i(x)$, $\nabla G_j(x)$ and $\nabla f(x)$ satisfying that $\forall x$, $\forall y$, $\forall i \in \{1,\cdots,n_1 \}$, $\forall j \in \{1,\cdots,n_2 \}$:
	\begin{eqnarray}\label{ass3_1}
		\| \nabla F_i(x) - \nabla F_i(y) \|  &\leq &L_F \|x - y \|,\\
		\label{ass3_2} \| \nabla G_j(x) - \nabla G_j(y) \|  &\leq& L_G \|x - y \|, \\
		\label{ass3_3} \|\left ( \nabla  G_{j}(x) \right )^T \nabla  F_{i}( {G}(x)) &-& \left ( \nabla  G_{j}(y) \right )^T \nabla  F_{i}( {G}(y)) \|  \nonumber \\
		&\leq&  L_f \|x - y \|.
	\end{eqnarray}
	As proved  in \cite{lian2016finite}, according to (\ref{ass3_3}), we have:
	\begin{eqnarray}\label{ass3_5}
		\|\nabla f(x) - \nabla f(y) \| & \leq & L_f \left \| x-y \right \|, \forall x, \forall y.
	\end{eqnarray}
	Equivalently, (\ref{ass3_5})  can  also be written as follows:
	\begin{eqnarray} \label{ass3_4}
		f(x) \leq f(y) + \langle \nabla f(y) , x-y \rangle + \frac{L_f}{2}  \left \| x-y \right \|^2, \forall x, \forall y
	\end{eqnarray}
\end{assumption}

\noindent \textbf{Bounded gradients:} We assume that the gradients $\nabla  F_{i}(x)$ and $\nabla  G_{j}(x)$ are upper bounded.
\begin{assumption}\label{ass_bounded gradients}
	The gradients $\nabla  F_{i}(x)$ and $\nabla  G_{j}(x)$ have upper bounds $B_F$ and $B_G$ respectively.
	\begin{eqnarray}\label{ass4_1}
		\left \| \nabla  F_{i}(x)\right \| &\leq& B_F, \forall x, \forall i \in \{1,\cdots,n_1 \}
		\\ \label{ass4_2} \left \| \nabla  G_{j}(x)\right \| &\leq& B_G, \forall x, \forall j \in \{1,\cdots,n_2 \}
	\end{eqnarray}
\end{assumption}

Note that we do not need the strong convexity assumption when we analyze the convergence of VRSC-PG for the general problem.

\subsection{Strongly Convex Problem}
In this section, we prove that our VRSC-PG admits linear convergence rate for strongly convex finite-sum composition problem with nonsmooth penalty regularization (\ref{formulation1}). We need Assumptions 1, 2 and 3 in this section. Unlike Prox-SVRG in \cite{xiao2014proximal}, the estimated $ {v}_{t}^{s+1} $ is biased, i.e., $ \mathbb{E}_{I_t,A_t,B_t}[{v}_{t}^{s+1}] \neq \nabla f(x^{s+1}_t)$.
It makes the theoretical analysis for proving the convergence rate of VRSC-PG  more challenging than the analysis in \cite{xiao2014proximal}. In spite of this, we can demonstrate that $\mathbb{E}\|v_t^{s+1} - \nabla f(x_t^{s+1})\|^2$ is upper bounded as well.
\begin{lemma}\label{lem1001}
	Let $x^*$ be the optimal solution to problem (\ref{formulation1}) $H(x)$ such that $x^* = \arg\min_{x\in \mathbb{R}^N} H(x) $. We define $\gamma =  \biggl( \frac{64}{\mu}\left ( \frac{B_F^2L_G^2}{B} + \frac{B_G^4L_F^2}{A} \right ) + 8L_f  \biggr) $.  Supposing Assumptions \ref{assumption2}, \ref{definition1} and \ref{ass_bounded gradients} hold,  from the definition of $v_t^{s+1}$ in (\ref{alg3_1}), the following inequality holds that:
	\small	\begin{eqnarray}
		\mathbb{E}\|v_t^{s+1} - \nabla f(x_t^{s+1})\|^2  \leq  \gamma \biggl[H(x_{t}^{s+1}) - H(x^*)  + H(\tilde{x}^{s} ) - H(x^*) \biggr]
	\end{eqnarray}
\end{lemma}	
Therefore, when $x_t^{s+1}$ and $\tilde{x}^s$ converges to $x^*$, 	$\mathbb{E}\|v_t - \nabla f(x_t^{s+1})\|^2 $ also converges to zero. Thus, we can keep learning rate constant, and obtain faster convergence.

\begin{theorem} \label{AsySPSAGA_theorem1}
	Suppose Assumptions \ref{assumption2}, \ref{definition1} and \ref{ass_bounded gradients} hold. We let $x^* = \arg\min_{x\in \mathbb{R}^N} H(x)$, if $m$, $A$, $B$ and  $\eta$ are selected properly so that $\rho < 1$, where $\rho$ is defined as follows:
	{\small
		\begin{eqnarray}
			\rho =  \frac{  \frac{2 }{\mu} + 2 \eta  \left ( 6 \eta L_f + \left ( \frac{\eta}{2} +  \frac{4 }{\mu} \right ) \frac{32 }{\mu}\left ( \frac{B_F^2L_G^2}{B} + \frac{B_G^4L_F^2}{A} \right )  \right ) (m+1) }{2 \eta \left (\frac{7}{8} -\left ( 6 \eta L_f + \left ( \frac{\eta}{2} +  \frac{4 }{\mu} \right ) \frac{32 }{\mu}\left ( \frac{B_F^2L_G^2}{B} + \frac{B_G^4L_F^2}{A} \right ) \right ) \right )m}
		\end{eqnarray}
	}
	we can prove that our VRSC-PG admits linear convergence rate:
	\begin{eqnarray}
		\mathbb{E}H(\tilde{x}^{S}) -H(x^*)
		& \leq&   \rho^S \biggl( \mathbb{E} H(\tilde{x}^{0}) -H({x}^*) \biggr)
	\end{eqnarray}
\end{theorem}

As per Theorem \ref{AsySPSAGA_theorem1},  we need to choose $\eta$, $m$, $A$ and $B$ properly to make $\rho < 1$. We provide an example to show how to select these parameters.

\begin{corollary} \label{AsySPSAGA_theorem2}
	According to Theorem \ref{AsySPSAGA_theorem1},  we  set $\eta$, $m$, $A$ and $B$ as follows:
	\begin{eqnarray}
		\eta &=& \frac{1}{96 L_f} \hspace{4cm} \\
		m &=&  16 \left ( 1 +  \frac{ 96 L_f}{\mu } \right ) \\
		A&=&  \frac{2048  B_G^4L_F^2}{\mu^2 }     \\
		B&=&  \frac{2048  B_F^2L_G^2}{ \mu^2 }
	\end{eqnarray}
	we have the following linear convergence rate for VRSC-PG:
	\begin{eqnarray}
		\mathbb{E}H(\tilde{x}^{S}) -H(x^*)
		\leq   \left(\frac{2}{3}\right)^S \bigg(\mathbb{E} H(\tilde{x}^{0}) -H({x}^*) \biggr)
	\end{eqnarray}
\end{corollary}

\begin{remark}
	According to Theorem \ref{AsySPSAGA_theorem1}, in order to obtain $\mathbb{E} H(\tilde{x}^s) - H(x^*)  \leq \varepsilon$, the number of stages $S$ is required to satisfy:
	\begin{eqnarray}
		S &\geq &  \log \frac{\mathbb{E} H(\tilde{x}^0) - H(x^*)}{\varepsilon}  / \log \frac{1}{\rho}
	\end{eqnarray}
\end{remark}

As per Algorithm \ref{algorithmSVRG2}  and the definition of Sampling Oracle in \cite{wang2016accelerating},  to make the objective value gap $\mathbb{E}H(\tilde{x}^{s}) - H(x^*)  \leq \varepsilon$, the total query complexity we need to take is $O\biggl(\big(n_1 + n_2 + m(A+B + b_1)\big) \log (\frac{1}{\varepsilon})\biggr) = O\biggl((n_1 + n_2 + \kappa^3) \log (\frac{1}{\varepsilon})\biggr) $, where we let $\kappa = \max \biggl\{\frac{L_f}{\mu}, \frac{L_F}{\mu}, \frac{L_G}{\mu} \biggr\}$ and $b_1$ can be smaller than or proportional to $\kappa^2$.  It is better  than com-SVR-ADMM\cite{svradmm}  whose total query complexity is $O\biggl((n_1 + n_2 + \kappa^4) \log (\frac{1}{\varepsilon})\biggr) $.

\subsection{ General Problem}
In this section, we prove that VRSC-PG admits a  sublinear convergence rate $O(T^{-1})$ for the general finite-sum composition problem with nonsmooth regularization penalty.  It is much better than the state-of-the-art method ASC-PG \cite{wang2016accelerating} whose optimal convergence rate is $O(T^{-1/2})$. In this section, we only need Assumption 2 and 3.
The unbiased $ {v}_{t}^{s+1} $ makes our analysis nontrivial and it is much different from previous analysis for finite-sum problem \cite{reddi2016fast}. In our proof, we define
$
\mathcal{G}_{\eta}(x) = \frac{1}{\eta} \left(x - \text{Prox}_{\eta h(.)} (x - \nabla f(x)\right)
$.
\begin{theorem}\label{thm3}
	Suppose Assumptions \ref{definition1} and \ref{ass_bounded gradients} hold. Let $x^*$ be the optimal solution to problem (\ref{formulation1}), we have $x^* = \arg\min_{x\in \mathbb{R}^N} H(x)$. If $m$, $A$, $B$, $b_1$ and $\eta$ are selected properly such that:
	\begin{eqnarray}
		\label{thm3_cond}
		4 \biggl( \frac{\eta m^2 L_f^2}{b_1}+   \frac{2\eta m^2 B_G^4 L_F^2}{A} 	+  \frac{2 \eta m^2 B_F^2 L_G^2 }{B}   \biggr)  + \frac{L_f}{2}  \leq  \frac{1}{2\eta},
	\end{eqnarray}
	then the following inequality  holds that:
	\begin{eqnarray}
		\mathbb{E} \| \mathcal{G}_{\eta}(x_a) \|^2 & \leq & \frac{2}{(1- 2\eta L_f)\eta}  \frac{H(\tilde x^0) - H(x^*)}{T}
	\end{eqnarray}
	where $x_a$ is uniformly selected from $\{\{x_t^{s+1}\}^{m-1}_{t=0} \}_{t=0}^{S-1}$ and $T$ is a multiple of $m$,
\end{theorem}

As per Theorem \ref{thm3},  we need to choose $m$, $A$, $B$, $b_1$ and $\eta$ appropriately to make condition (\ref{thm3_cond}) satisfied. We provide an example to show how to select these parameters.

\begin{corollary}
	According to Theorem \ref{thm3}, we let $m=\left \lfloor (n_1+n_2)^{\frac{1}{3}} \right \rfloor $, $\eta = \frac{1}{4L_f}$, $b_1 = (n_1+n_2)^{\frac{2}{3}}$ and $T$ be a multiple of $m$, it is easy to know that if $A$ and  $B$  are lower bounded:
	\begin{eqnarray}
		A & \geq & \frac{8 m^2 B_G^4 L_F^2 }{L_f} \\
		B & \geq & \frac{8m^2 B_F^2 L_G^2}{L_f}
	\end{eqnarray}
	we can obtain  sublinear convergence rate for VRSC-PG:
	\begin{eqnarray}
		\mathbb{E} \| \mathcal{G}_{\eta}(x_a) \|^2 & \leq & {16L_f } \frac{H(\tilde x^0) - H(x^*)}{T}
	\end{eqnarray}
\end{corollary}

As per Algorithm \ref{algorithmSVRG2}  and the definition of Sampling Oracle in \cite{wang2016accelerating},  to obtain $\varepsilon$-accurate solution,  $	 \mathbb{E} \| \mathcal{G}_{\eta}(x_a) \|^2   \leq \varepsilon$, the total query complexity we need to take is $O(n_1 + n_2 +  \frac{A + B + b_1}{\varepsilon}) =O\biggl( n_1 + n_2+ \frac{(n_1+n_2)^{{2}/{3}}}{\varepsilon}\biggr) $, where $A$, $B$ and $b_1$ are proportional to $(n_1+n_2)^{\frac{2}{3}} $. Therefore, our method improves  the state-of-the-art convergence rate of stochastic composition optimization for general problem from $O(T^{-1/2})$ (Optimal convergence rate for ASC-PG) to $O\biggl( (n_1+n_2)^{{2}/{3}}T^{-1}\biggr)$.

\begin{figure*}[!t]
	\centering
	\begin{subfigure}[b]{0.45\textwidth}
		\centering
		\includegraphics[width=3in]{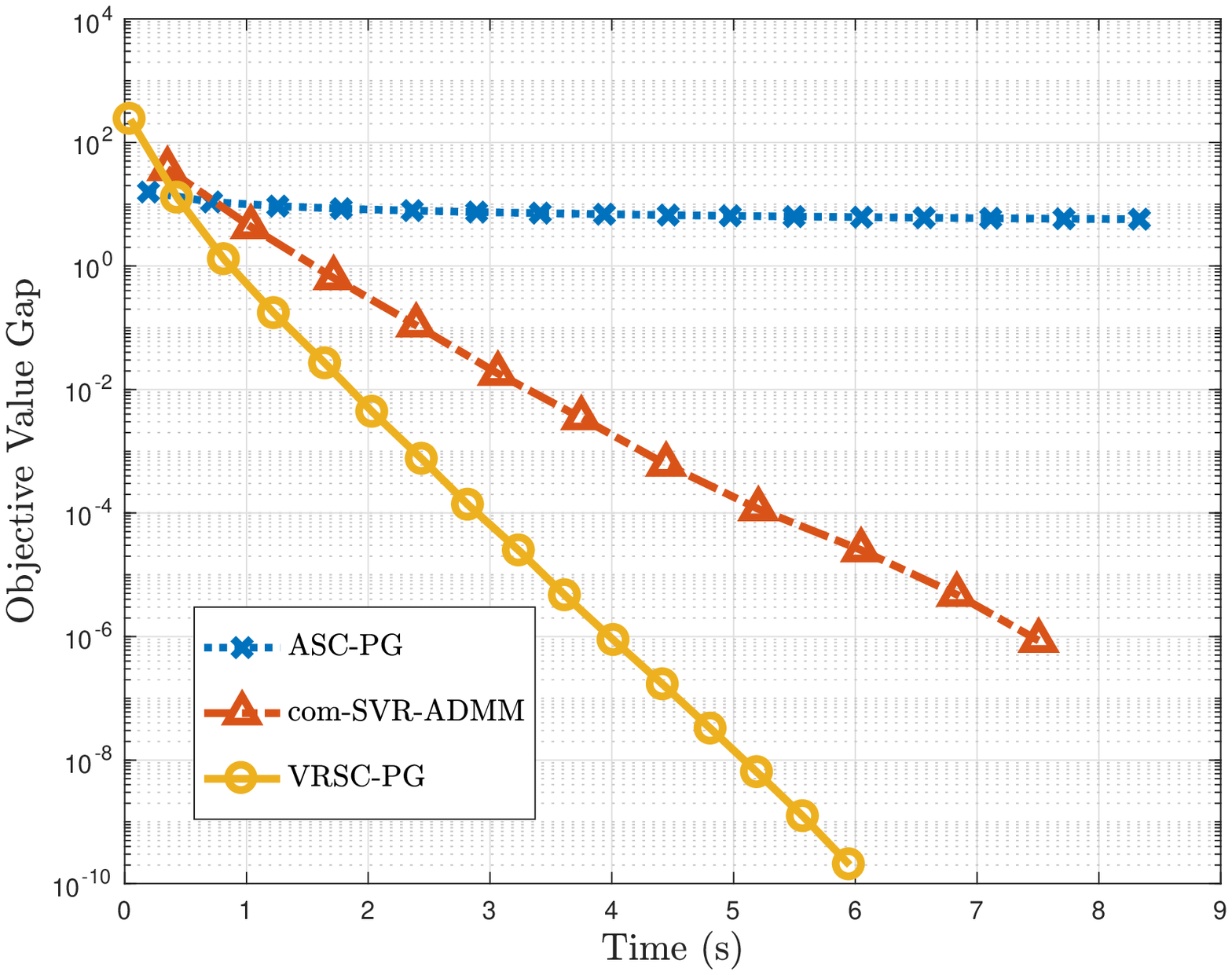}
		\caption{$\kappa_{cov}=2$}
	\end{subfigure}
	\begin{subfigure}[b]{0.45\textwidth}
		\centering
		\includegraphics[width=3in]{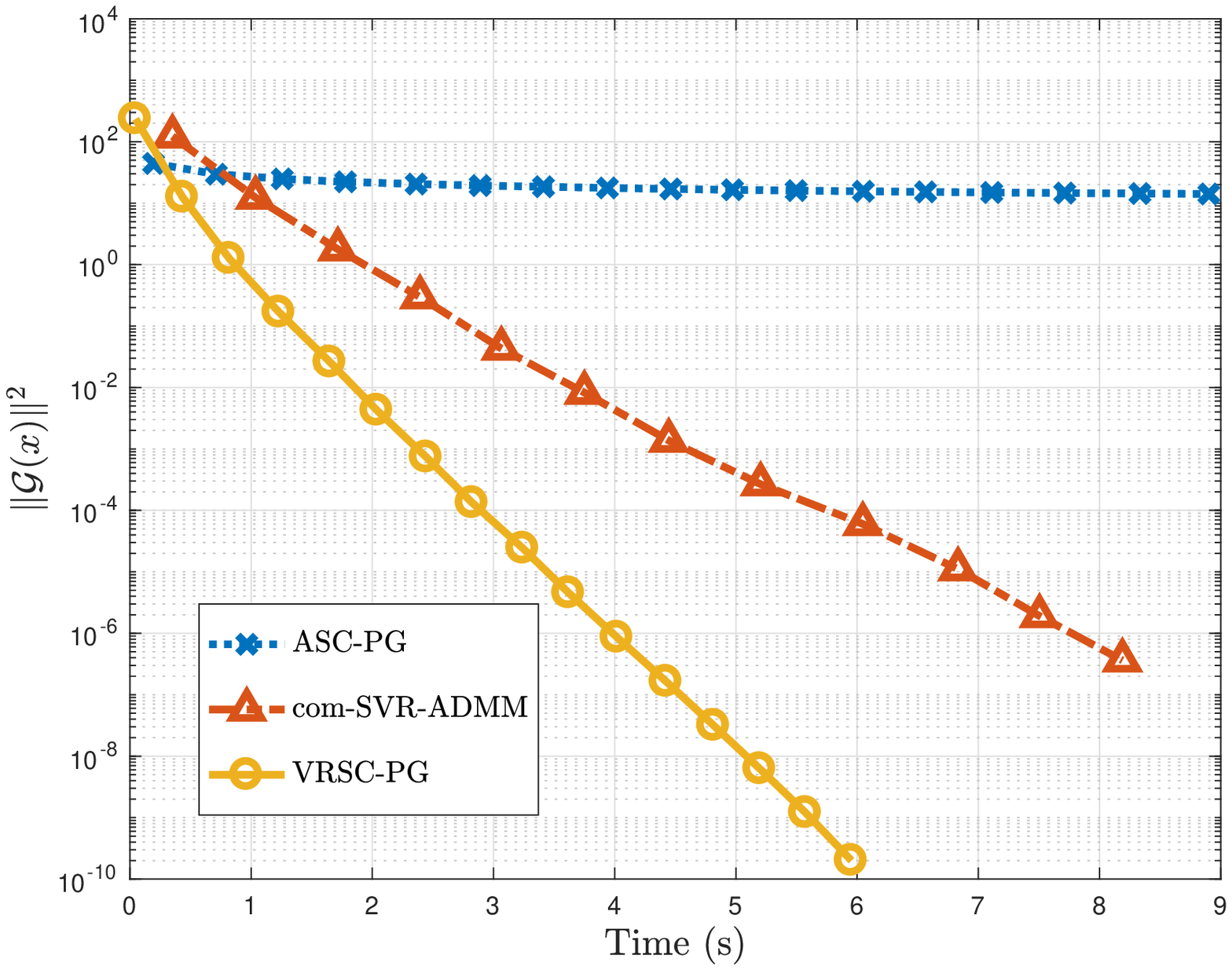}
		\caption{$\kappa_{cov}=2$}
	\end{subfigure}
	\centering
	\begin{subfigure}[b]{0.45\textwidth}
		\centering
		\includegraphics[width=3in]{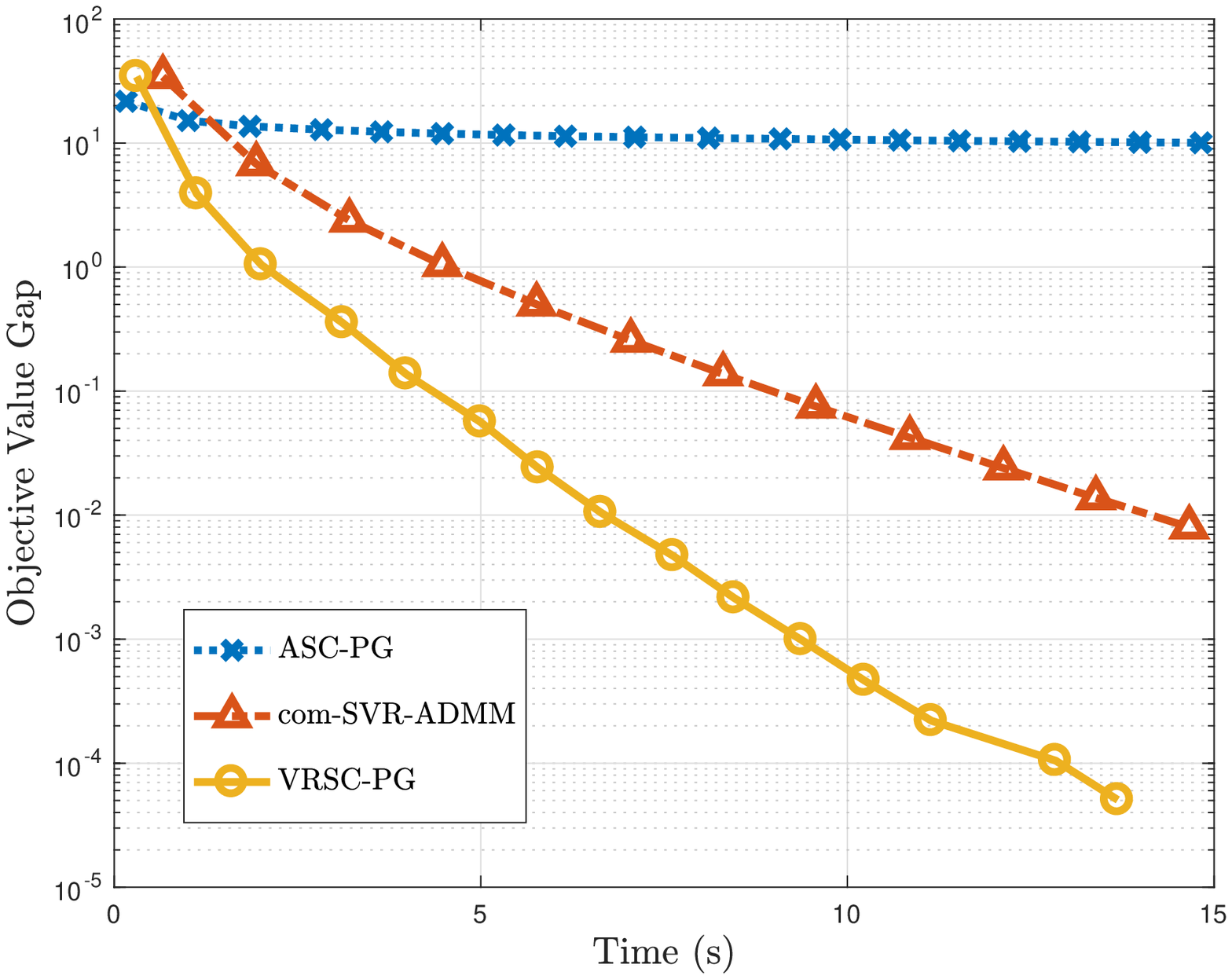}
		\caption{$\kappa_{cov}=10$}
	\end{subfigure}
	\begin{subfigure}[b]{0.45\textwidth}
		\centering
		\includegraphics[width=3in]{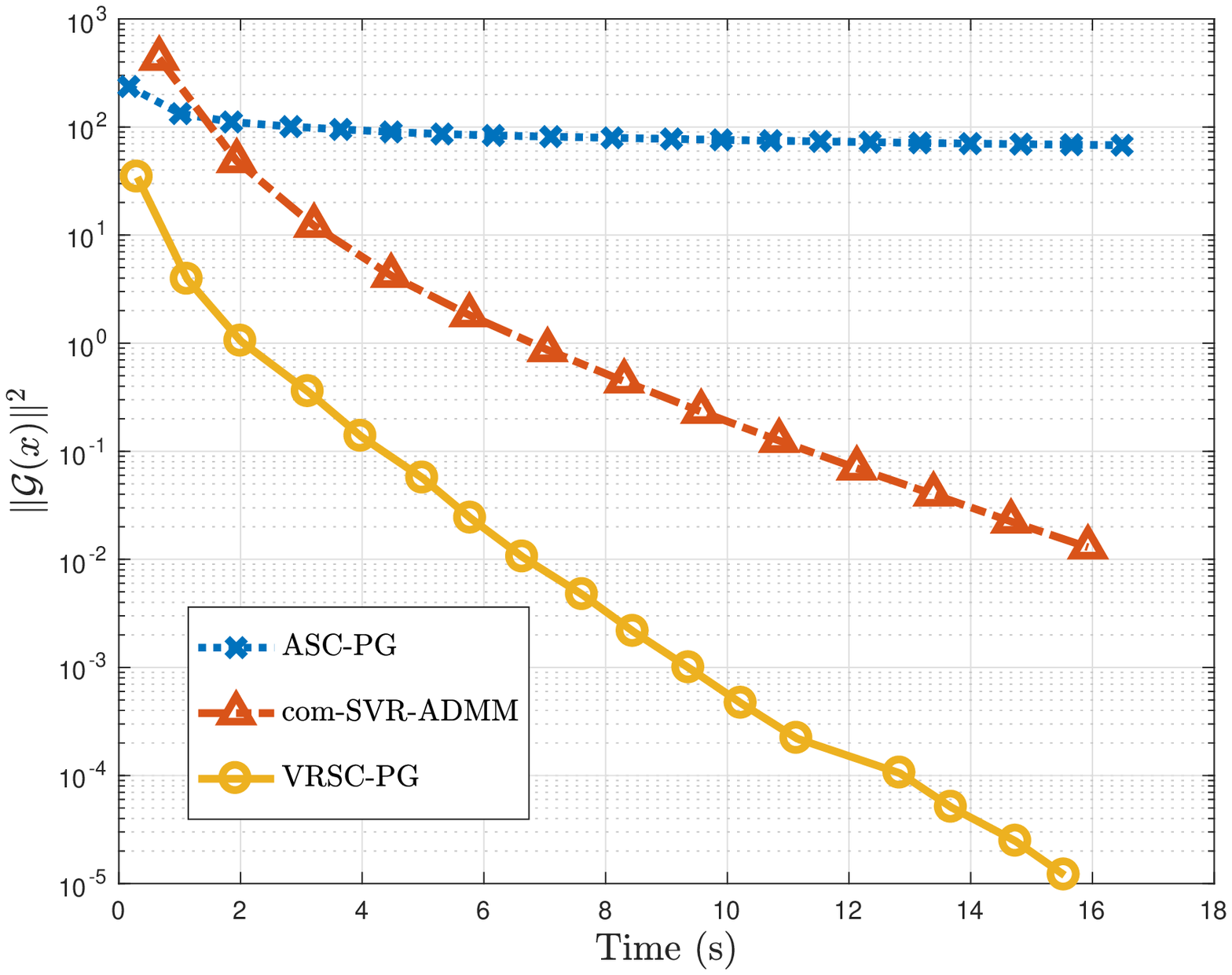}
		\caption{$\kappa_{cov}=10$}
	\end{subfigure}
	\caption{Experimental results for meaning-variance portfolio management on synthetic data. $\kappa_{cov}$
		is the conditional number of the covariance matrix of the corresponding Gaussian distribution which is used to generate reward. We use time as $x$ axis, and it is proportional to the query complexity. In $y$ axis,  the objective value gap is defined as $H(x) - H(x^*) $, where $x^*$ is obtained by running our methods for enough iterations until convergence. $\| \mathcal{G}(x) \|^2$ denotes the $\ell_2$-norm of the full gradient, where $\mathcal{G}(x) = \nabla f(x) + \partial h(x)$. }
	\label{exp1_fig}
\end{figure*}

\section{Experimental Results}
We conduct two experiments to evaluate our proposed method: (1) application to portfolio management; (2) application to policy evaluation in reinforcement learning.

In the experiments, we compare our proposed VRSC-PG with two other related methods:
\begin{itemize}
	\item Accelerated stochastic compositional proximal gradient (ASC-PG)  \cite{wang2016accelerating};
	\item  Stochastic variance reduced ADMM for Stochastic composition optimization (com-SVR-ADMM) \cite{svradmm};
\end{itemize}
In our experiments, learning rate $\eta$ is tuned from $\{1, 10^{-1}, 10^{-2}, 10^{-3}, 10^{-4}\}$. We keep the learning rate constant for  com-SVR-ADMM and VRSC-PG in the optimization. For  ASC-PG, in order to guarantee convergence,  learning rate is decreased as per $ \frac{\eta}{1+t}$, where $t$ denotes the number of iterations.

\begin{figure*}[!t]
	\centering
	\begin{subfigure}[b]{0.45\textwidth}
		\centering
		\includegraphics[width=3in]{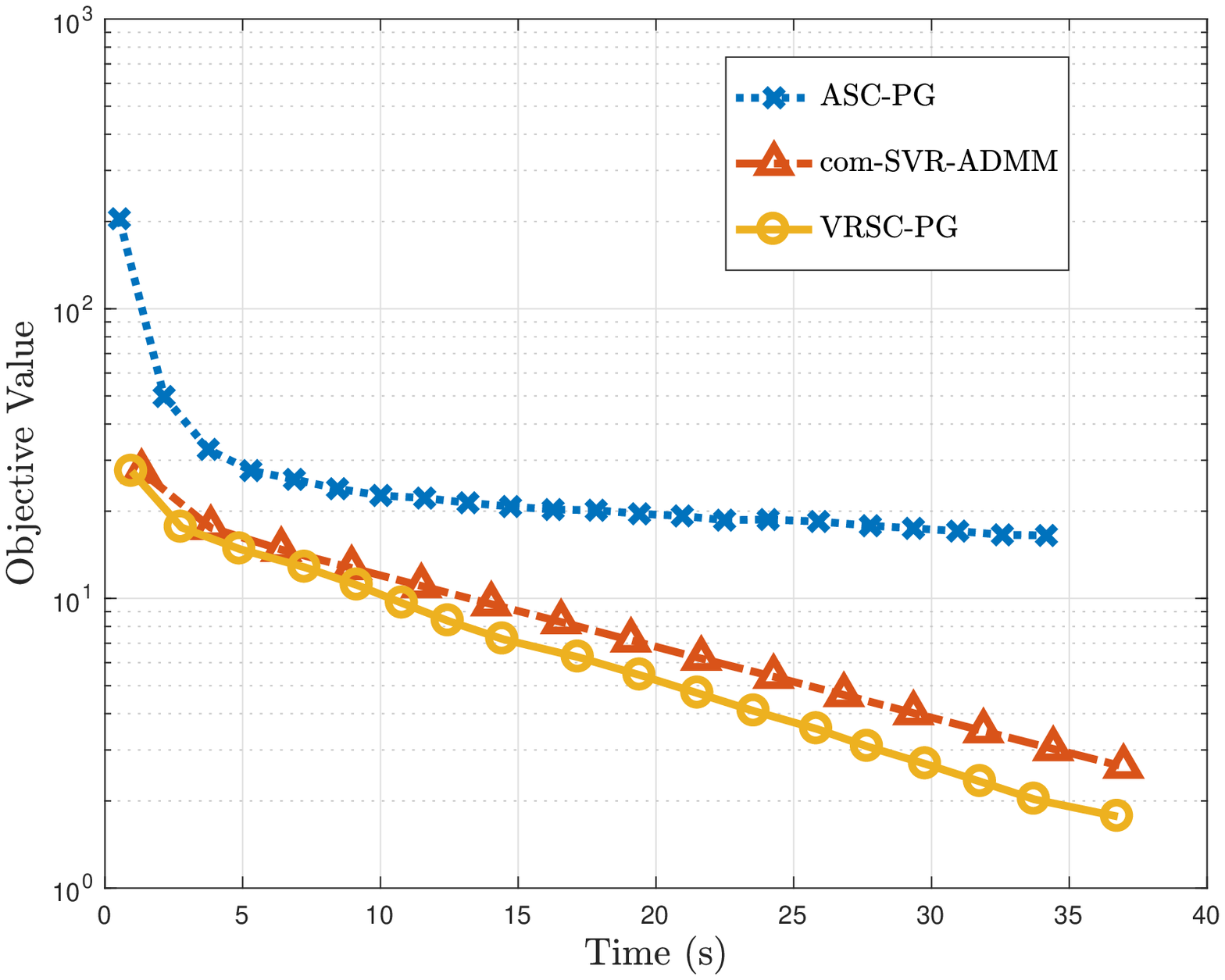}
	\end{subfigure}
	\begin{subfigure}[b]{0.45\textwidth}
		\centering
		\includegraphics[width=3in]{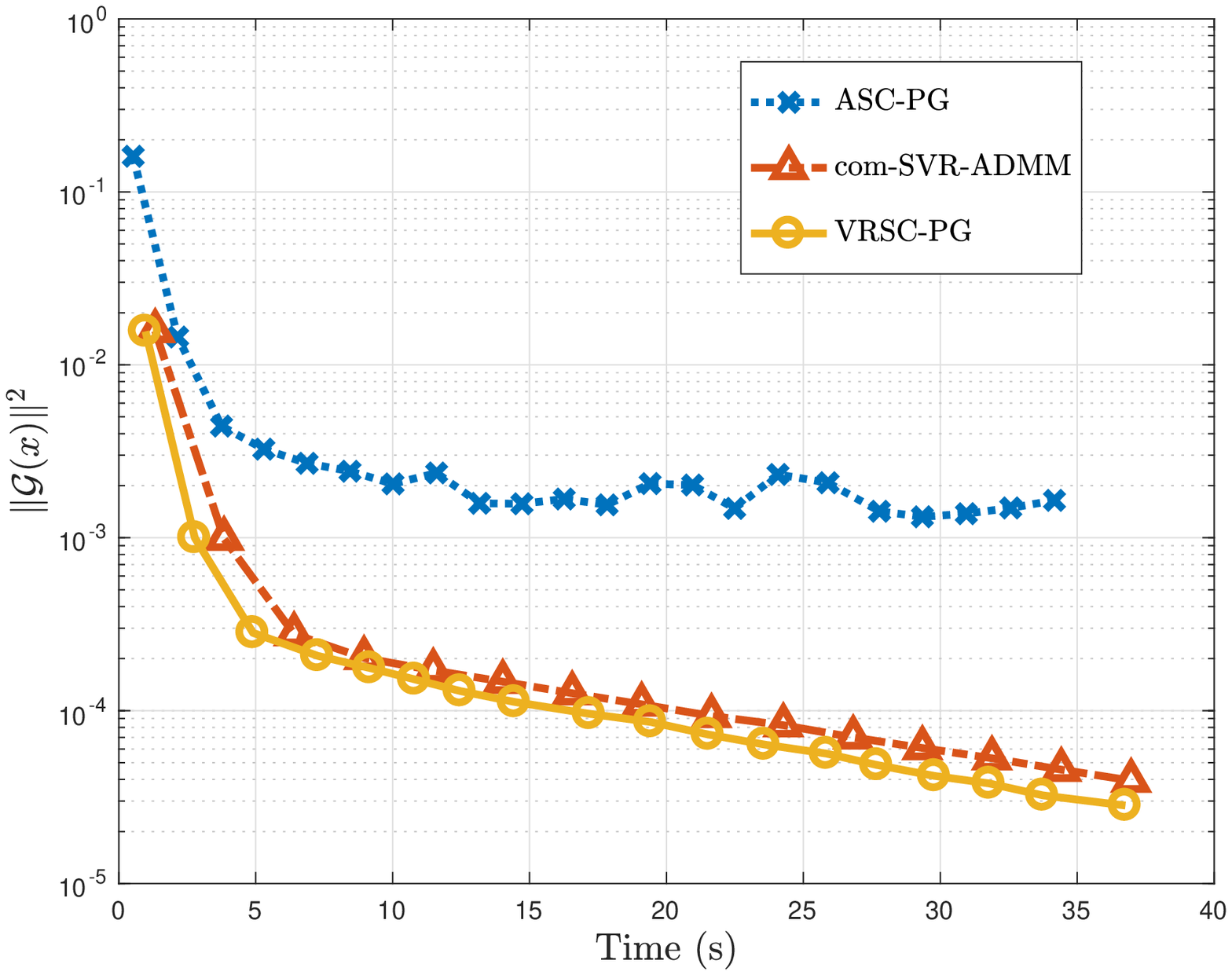}
	\end{subfigure}
	\caption{Figures show the experimental results of policy evaluation in reinforcement learning. We plot the convergence of objective value  and the full gradient $\|\mathcal{G}(x)\|^2$  regarding  time respectively. $\| \mathcal{G}(x) \|^2$ denotes the $\ell_2$-norm of the full gradient, where $\mathcal{G}(x) = \nabla f(x) + \partial h(x)$. }
	\label{exp2_fig}
\end{figure*}
\subsection{Application to Portfolio Management}
Suppose there are $N$ assets we can invest, $r_t \in \mathbb{R}^N$ denotes the rewards of $N$ assets at time $t$.  Our goal is to maximize the return of the investment and to minimize the risk of the investment at the same time. Portfolio management problem can be formulated as the mean-variance optimization as follows:
\begin{eqnarray}
	\label{exp1_obj}
	\min\limits_{x \in \mathbb{R}^N} -\frac{1}{n} \sum\limits_{t=1}^n \left<r_t, x\right> + \frac{1}{n} \sum\limits_{t=1}^n \bigg( \left< r_t, x \right> - \frac{1}{n}\sum\limits_{j=1}^n \left<r_j, x \right>  \bigg)^2
\end{eqnarray}
where $x \in \mathbb{R}^N$ denotes the investment quantity vector in $N$ assets.  According to  \cite{lian2016finite}, problem (\ref{exp1_obj}) can also be viewed as the composition problem as (\ref{formulation1}). 
In our experiment, we also add a nonsmooth regularization penalty $h(x) = \lambda |x|$ in the mean-variance optimization problem (\ref{exp1_obj}).

Similar to the experimental settings in \cite{lian2016finite}, we let $n=2000$ and $N=200$. Rewards $r_t$ are generated in two steps: (1) Generate a Gaussian distribution on $\mathbb{R}^N$, where we define the condition number of its covariance matrix as $\kappa_{cov}$. Because $\kappa_{cov}$ is proportional to $\kappa$, in our experiment, we will control $\kappa_{cov}$ to change the value of $\kappa$; (2) Sample rewards $r_t$  from the Gaussian distribution and make all elements positive to guarantee that this problem has a solution. In the experiment, we compared three methods on two synthetic datasets, which are generated through Gaussian distributions with $\kappa_{cov} = 2$ and $\kappa_{cov}=10$ separately. We set $\lambda = 10^{-3}$ and $A=B=b_1 =5$. We just select the values of $A,B,b_1$ casually, it is probable that we can get better results as long as we tune them carefully.

Figure \ref{exp1_fig} shows the convergence of compared methods regarding time. We suppose that the elapsed time is proportional to the query complexity. Objective value gap means $H(x_t) - H(x^*)$, where $x^*$ is the optimal solution to $H(x)$. We compute $H(x^*)$ by running our method until convergence. Firstly, by observing the $x$ and $y$ axises in Figure \ref{exp1_fig}, we can know that when $\kappa_{cov} = 10$, all compared methods need more time to minimize problem (\ref{exp1_obj}), which is consistent with our analysis. Increasing $\kappa$ will increase the total query complexity. Secondly, we can also find out that com-SVR-ADMM and VRSC-PG admit linear convergence rate.  ASC-PG runs faster at the beginning, because of their low query complexity in each iteration. However, their convergence slows down when the learning rate gets small. In four figures, our SVRC-PG always has the best performance compared to other compared methods.

\subsection{Application to Reinforcement Learning}
We then apply stochastic composition optimization to reinforcement learning and evaluate three compared methods in the task of policy evaluation. In reinforcement learning, let $ V^{\pi}(s)$ be the value of state $s$ under  policy $\pi$. The value function $ V^{\pi}(s)$  can be evaluated through Bellman equation as follows:
\begin{eqnarray}
	\label{bellman}
	V^{\pi}(s_1) = \mathbb{E} [r_{s_1,s_2} + \gamma V^{\pi} (s_2)|s_1  ]
\end{eqnarray}
for all $ s_1,s_2 \in \{1,2,...,S\}$, where $S$ represents the number of total states.   According to \cite{wang2016accelerating}, the Bellman equation (\ref{bellman}) can also be written as a composition problem. In our experiment, we also add sparsity regularization $h(x) = \lambda |x|$ in the objective function.

Following \cite{dann2014policy}, we generate a Markov decision process (MDP). There are $400$ states and $10$ actions at each state. The transition probability is generated randomly from the uniform distribution in the range of  $[0,1]$. We then add $10^{-5}$ to each element of transition matrix to ensure the ergodicity of our MDP. The rewards $r(s,s')$ from state $s$ to state $s'$ are also sampled  uniformly in the range of $[0, 1]$. In our experiment, we set $\lambda = 10^{-3}$ and $A=B=b_1 =5$. We also select these values casually, better results can be obtained if we tune them carefully.

In Figure \ref{exp2_fig}, we plot the convergence of the objective value and $\|\mathcal{G}(x)\|^2$ in terms of time.  We can observe that  VRSC-PG is much faster than ASC-PG, which has been reflected in the analysis of convergence rate already. It is also obvious that our VRSC-PG converges faster than com-SVR-ADMM. Experimental results on policy evaluation also verify our theoretical analysis.

\section{Conclusion}
In this paper, we propose variance reduced stochastic compositional proximal gradient method (VRSC-PG) for composition problem with nonsmooth regularization penalty. We also analyze the convergence rate of our method: (1) for strongly convex composition problem, VRSC-PG is proved to admit linear convergence; (2) for general composition problem, VRSC-PG significantly improves the state-of-the-art convergence rate from $O(T^{-1/2})$ to $O((n_1+n_2)^{{2}/{3}}T^{-1})$. Both of our theoretical analysis, to the best of our knowledge, are the state-of-the-art results for stochastic composition optimization. Finally, we apply our method to two different applications, portfolio management and reinforcement learning. Experimental results show that our method always has the best performance in different cases and verify the conclusions of theoretical analysis.

\bibliographystyle{abbrv}
\bibliography{composite}  

\appendix

\section{Strongly Convex Problem }

\textbf{Proof to Lemma \ref{lem1001}}
\begin{proof}
	Following \cite{xiao2014proximal}, we define function $\phi_i(x)$:
	\begin{eqnarray}
		\phi_i(x) = F_i(G(x)) - F_i(G(x^*)) - \left<(\nabla G(x^*))^T \nabla F_i(G(x^*)),  (x - x^*) \right>  
	\end{eqnarray}
	where $x^* = \arg\min_x H(x)$ denotes the optimal solution of $H(x)$.  It is easy to know that $\nabla \phi_i(x^*)  = 0$ and $\min_x \phi_i(x) = \phi_i(x^*)=0$. For any $x, y \in \mathbb{R}^N$, we have: 
	\begin{eqnarray}
		&& \| \nabla \phi_i(x) - \nabla \phi_i(y)  \|  \nonumber \\
		& = & \| (\nabla G(x))^T \nabla F_i(G(x)) - (\nabla G(y))^T\nabla F_i(G(y)) \| \nonumber \\
		&= & \| \frac{1}{n_2} \sum\limits_{j=1}^{n_2} (\nabla G_j(x) )^T F_i(G(x)) -  \frac{1}{n_2} \sum\limits_{j=1}^{n_2} (\nabla G_j(y) )^T F_i(G(y))  \|  \nonumber \\
		&\leq & \frac{1}{n_2} \sum\limits_{j=1}^{n_2}  \|  (\nabla G_j(x) )^T F_i(G(x)) -   (\nabla G_j(y) )^T F_i(G(y))  \| \nonumber \\
		&\leq & L_f \| x - y  \|
	\end{eqnarray}
	where the first inequality follows from the triangle inequality, and the second inequality follows from (\ref{ass3_3}) in Assumption \ref{ass_bounded gradients}. Therefore, $\nabla \phi_i(x)$ is $L_f$-Lipschitz continuous. According to Theorem 2.1.5 in  \cite{nesterov2013introductory}, we have:
	\begin{eqnarray}
		\frac{1}{2L_f} \|\nabla \phi_i (x) - \nabla \phi_i(x^*)\|^2 &\leq & \phi_i (x) - \phi_i(x^*) -   \left< \nabla \phi_i(x^*) , x-x^* \right>
	\end{eqnarray}
	Because $\phi_i (x^*) = 0$ and $\nabla \phi_i(x^*) = 0$, it follows that:
	\begin{eqnarray}
		&& \frac{1}{2L_f} \| (\nabla G(x))^T  \nabla F_i(G(x) ) - (\nabla G(x^*))^T \nabla F_i(G(x^*)) \|^2  \nonumber \\
		& \leq & F_i(G(x)) - F_i(G(x^*)) - \left< (\nabla G(x^*))^T \nabla F_i(G(x^*)),  (x - x^*) \right>  
	\end{eqnarray}
	Because $x^* = \arg\min_x f(x) + h(x) $, there exists $\xi^* \in  \partial h(x^*) $ so that $ (\nabla G(x^*))^T \nabla F(G(x^*)) + \xi^* =0 $. Then it holds that:
	\begin{eqnarray}
		\label{iq6001}
		&& \frac{1}{ n_1} \sum\limits_{i=1}^{n_1}\| (\nabla G(x))^T \nabla F_i(G(x) ) - (\nabla G(x^*))^T \nabla F_i(G(x^*)) \|^2  \nonumber \\
		&\leq & 2L_f \bigg[  F(G(x)) - F(G(x^*)) - \left< (\nabla G(x^*))^T \nabla F(G(x^*)),  (x - x^*) \right>   \bigg] \nonumber \\
		& = & 2L_f \bigg[  F(G(x)) - F(G(x^*)) + \left<\xi^* ,  (x - x^*) \right>   \bigg] \nonumber \\
		&\leq & 2L_f \bigg[ F(G(x)) - F(G(x^*)) + h(x) - h(x^*)  \bigg] \nonumber \\
		&\leq & 2L_f \bigg[ H(x) - H(x^*)  \bigg]
	\end{eqnarray}
	where the third inequality is from the convexity of $h(x)$.  We let $u_t$ be the unbiased estimate of  $\nabla f(x_t^{s+1})$:
	\begin{eqnarray}
		{u}_{t}^{s+1} = \frac{1}{b_1} \sum\limits_{i_t \in I_t} \biggl( \left ( \nabla  G({x}_{t}^{s+1}) \right )^T \nabla  F_{i_t}( {G}(x_t^{s+1}))  - \left ( \nabla  G(\tilde{x}^{s}) \right )^T \nabla  F_{i_t}(\tilde{G}^{s}) \biggr)+ \nabla  f(\tilde{x}^{s})
	\end{eqnarray}
	where $\mathbb{E}  {u}_{t}^{s+1} = \nabla f( {x}_t^{s+1})$.  We can get the upper bound of the variance of $u_t$:
	\begin{eqnarray}\label{lemm1_1}
		&&	\mathbb{E}\|u_t - \nabla f(x_t^{s+1})\|^2 \nonumber\\
		&	 =& \mathbb{E} \|  \frac{1}{b_1} \sum\limits_{i_t \in I_t} \biggl( \left ( \nabla  G({x}_{t}^{s+1}) \right )^T \nabla  F_{i_t}( {G}(x_t^{s+1}))  - \left ( \nabla  G(\tilde{x}^{s}) \right )^T \nabla  F_{i_t}(\tilde{G}^{s}) + \nabla  f(\tilde{x}^{s}) - \nabla f(x_t^{s+1}) \biggr)  \|^2 \nonumber \\
		& =&  \frac{1}{b_1}   \sum\limits_{i_t \in I_t} \mathbb{E}\|  \left ( \nabla  G({x}_{t}^{s+1}) \right )^T \nabla  F_{i_t}( {G}(x_t^{s+1}))  - \left ( \nabla  G(\tilde{x}^{s}) \right )^T \nabla  F_{i_t}(\tilde{G}^{s}) + \nabla  f(\tilde{x}^{s}) - \nabla f(x_t^{s+1})   \|^2 \nonumber \\
		&\leq & \frac{1}{b_1}   \sum\limits_{i_t \in I_t}  \mathbb{E} \|  \left ( \nabla  G({x}_{t}^{s+1}) \right )^T \nabla  F_{i_t}( {G}(x_t^{s+1}))  - \left ( \nabla  G(\tilde{x}^{s}) \right )^T \nabla  F_{i_t}(\tilde{G}^{s})  \|^2 \nonumber \\
		&\leq & \frac{2}{b_1}   \sum\limits_{i_t \in I_t}  \biggl( \mathbb{E}\|  \left ( \nabla  G({x}_{t}^{s+1}) \right )^T \nabla F_{i_t} (G(x_t^{s+1})) -  \left ( \nabla  G({x^*}) \right )^T \nabla F_{i_t} ( G(x^*))  \|^2 \nonumber \\
		&& + 
		\mathbb{E}\|  \left ( \nabla  G({\tilde x^s}) \right )^T \nabla F_{i_t} ( G( \tilde{x}^{s})) -  \left ( \nabla  G({x^*}) \right )^T\nabla F_{i_t} (G(x^*))  \|^2 \biggr) \nonumber \\
		&\leq & 4 L_f\bigg[ H(x_{t}^{s+1} )  - H(x^*)  +  H(\tilde{x}^s) - H(x^*)  \bigg]
	\end{eqnarray}
	where the second equality follows from Lemma \ref{ex_lem2}, the last inequality follows from (\ref{iq6001})
	Then, we compute the upper bound of $\mathbb{E}  \left  \|  {u}_{t}^{s+1}  - {v}_{t}^{s+1} \right  \|^2$:
	\begin{eqnarray}\label{lemm1_2}
		\mathbb{E}  \left  \|  {u}_{t}^{s+1}  - {v}_{t}^{s+1} \right  \|^2  &=&     \mathbb{E}  \left  \|  \frac{1}{b_1}   \sum\limits_{i_t \in I_t}   \biggl(  \left ( \nabla  G({x}_{t}^{s+1}) \right )^T \nabla  F_{i_t}( {G}(x_t^{s+1}))  - \left ( \nabla \widehat{G}_t^{s+1} \right )^T \nabla  F_{i_t}( \widehat{G}_t^{s+1}) \biggr) \right  \|^2 \nonumber \\
		&\leq &
		\frac{1}{b_1}   \sum\limits_{i_t \in I_t}  \underbrace{  \mathbb{E}  \left  \|    \left ( \nabla  G({x}_{t}^{s+1}) \right )^T \nabla  F_{i_t}( {G}(x_t^{s+1}))  - \left ( \nabla \widehat{G}_t^{s+1} \right )^T \nabla  F_{i_t}( \widehat{G}_t^{s+1})  \right  \|^2 }_{T_1}\nonumber \\
	\end{eqnarray}
	where the inequality follows from Lemma \ref{ex_lem2}. Then we can get the upper bound of $T_3$ as follows:
	\begin{eqnarray}
		T_1 
		&\leq &  2\mathbb{E}  \left  \|   \left ( \nabla  G({x}_{t}^{s+1}) \right )^T \nabla  F_{i_t}( {G}(x_t^{s+1}))  - \left ( \nabla G(x_t^{s+1}) \right )^T \nabla  F_{i_t}( \widehat{G}_t^{s+1})  \right   \|^2 \nonumber \\
		&& +  2\mathbb{E}  \left  \|     \left ( \nabla G(x_t^{s+1}) \right )^T \nabla  F_{i_t}( \widehat{G}_t^{s+1})  - \left ( \nabla \widehat{G}_t^{s+1} \right )^T \nabla  F_{i_t}( \widehat{G}_t^{s+1})   \right  \|^2 \nonumber \\
		&\leq &	2B_F^2 \mathbb{E}  \left  \|   \left ( \nabla G(x_t^{s+1}) \right )^T   - \left ( \nabla \widehat{G}_t^{s+1} \right )^T \right  \|^2 + 2B_G^2\mathbb{E}  \left  \| \nabla  F_{i_t}( {G}(x_t^{s+1}))  -  \nabla  F_{i_t}( \widehat{G}_t^{s+1}) \right  \|^2 \nonumber \\
		&\leq &	2B_F^2 \mathbb{E}  \left  \|   \left ( \nabla G(x_t^{s+1}) \right )^T   - \left (  \nabla G(\tilde{x}^s) - \frac{1}{B} \sum\limits_{1 \leq j \leq B} ( \nabla G_{B_t[j]} (\tilde{x}^s)  - \nabla G_{B_t[j]} (x_t^{s+1}) )      \right )^T \right  \|^2 \nonumber \\
		&& + 2B_G^2 L_F^2\mathbb{E}  \left  \|  {G}(x_t^{s+1})  -   \left(G^s - \frac{1}{A} \sum\limits_{1\leq j \leq A } (G_{A_t[j]}(\tilde{x}^s) - G_{A_t[j]}(x_t^{s+1} ))  \right) \right  \|^2 \nonumber \\
		& =  &	\frac{2B_F^2}{B^2} \mathbb{E}  \left  \|    \sum\limits_{1 \leq j \leq B} \left( \nabla G_{B_t[j]} (\tilde{x}^s)  - \nabla G_{B_t[j]} (x_t^{s+1}) +   \nabla G(x_t^{s+1})    -  \nabla G(\tilde{x}^s)       \right )^T \right  \|^2 \nonumber \\
		&& + \frac{ 2B_G^2 L_F^2}{A^2} \mathbb{E}  \left  \|   \sum\limits_{1\leq j \leq A } \left(G_{A_t[j]}(\tilde{x}^s) - G_{A_t[j]}(x_t^{s+1} ) +{G}(x_t^{s+1})  -   G^s    \right) \right  \|^2 \nonumber \\
		&= &	\frac{2B_F^2}{B^2} \sum\limits_{1 \leq j \leq B} \mathbb{E}  \left  \|     \left( \nabla G_{B_t[j]} (\tilde{x}^s)  - \nabla G_{B_t[j]} (x_t^{s+1}) +   \nabla G(x_t^{s+1})    -  \nabla G(\tilde{x}^s)       \right )^T \right  \|^2 \nonumber \\
		&& + \frac{ 2B_G^2 L_F^2}{A^2} \sum\limits_{1\leq j \leq A } \mathbb{E}  \left  \|    \left(G_{A_t[j]}(\tilde{x}^s) - G_{A_t[j]}(x_t^{s+1} ) +{G}(x_t^{s+1})  -   G^s    \right) \right  \|^2 \nonumber \\
		&\leq &	\frac{2B_F^2}{B^2} \sum\limits_{1 \leq j \leq B}  8L_G^2 \left(  \mathbb{E}\| \tilde{x}^s  - x^* \|^2  + \mathbb{E} \|x_t^{s+1} - x^*\|^2 \right)  \nonumber \\
		&& + \frac{ 2B_G^2 L_F^2}{A^2} \sum\limits_{1\leq j \leq A } 8B_G \left(  \mathbb{E}\| \tilde{x}^s  - x^* \|^2  + \mathbb{E} \|x_t^{s+1} - x^*\|^2  \right) \nonumber \\
		& =  & \nonumber 16\left ( \frac{B_F^2L_G^2}{B} + \frac{B_G^4L_F^2}{A} \right ) \mathbb{E}  \left ( \|x^{s+1}_t-x^*\|^2 + \|\tilde x^{s}-x^*\|^2\right )
		\\  & \leq  &  \frac{32}{\mu}\left ( \frac{B_F^2L_G^2}{B} + \frac{B_G^4L_F^2}{A} \right )  \mathbb{E} \left ( H({x}_t^{s+1}) -H({x}^*) + H( \tilde{x}^{s}) -H({x}^*) \right )
		\label{fuck}
	\end{eqnarray}
	
	where the first inequality follows from Lemma \ref{ex_lem2}, the second, third fourth inequality follows from Assumption \ref{ass_bounded gradients}, the third equality follows from Lemma \ref{ex_lem1}, the last inequality follows from the strong convexity of $H(x)$, such that $ \| x- x^*\|^2 \leq \frac{2}{\mu}  (H(x) - H(x^*))$. Combine (\ref{lemm1_1}) and (\ref{lemm1_2}), then we have:
	\begin{eqnarray}
		&&\mathbb{E}\|v_t - \nabla f(x_t^{s+1})\|^2 \nonumber \\
		& \leq & 2\mathbb{E}  \left  \|  {u}_{t}^{s+1}  - {v}_{t}^{s+1} \right  \|^2  + 2	\mathbb{E}\|u_t - \nabla f(x_t^{s+1})\| \nonumber \\
		&\leq & \biggl( \frac{64}{\mu}\left ( \frac{B_F^2L_G^2}{B} + \frac{B_G^4L_F^2}{A} \right ) + {8L_f}  \biggr)  \biggl[H(x_{t}^{s+1}) - H(x^*)  + H(\tilde{x}^{s} ) - H(x^*) \biggr]
	\end{eqnarray}
	where the first inequality follows from Lemma  \ref{ex_lem2}.\\ \\
\end{proof}

\begin{lemma} \label{SVRG-FSC-2_lem2}
	Suppose all assumption hold, given $x^*$ is the optimal solution to problem $H(x)$, there exists a constant $\alpha >0$ such that:
	\begin{eqnarray}\label{SVRG-FSC-2_lem2_0}
		&& \mathbb{E} \left \langle  {v}_{t}^{s+1} - \nabla f( {x}_t^{s+1}) ,  x^* - x^{s+1}_{t+1}  \right \rangle
		\nonumber	\\  \nonumber & \leq  &   \left ( 6 \eta L_f + \left ( \frac{\eta}{2} +  \frac{1}{2\alpha} \right ) \frac{32 }{\mu}\left ( \frac{B_F^2L_G^2}{B} + \frac{B_G^4L_F^2}{A} \right )  \right ) \mathbb{E} \left ( H({x}_t^{s+1}) -H({x}^*) + H(\tilde{x}^{s}) -H({x}^*) \right )
		\\  &   &   + \frac{\alpha }{\mu} \mathbb{E} \left ( H({x}_{t+1}^{s+1}) -H({x}^*)\right )
	\end{eqnarray}
\end{lemma}
\begin{proof}
	Define $\bar{x}^{s+1}_{t+1}  =\textrm{Prox}_{ \eta h(\cdot) }\left ( x^{s+1}_t - {\eta}
	\nabla f( {x}_t^{s+1}) \right )$,  we have that
	\begin{eqnarray}\label{AsySPSAGA_lem2_1}
		&&\nonumber  \mathbb{E} \left \langle  {v}_{t}^{s+1} - \nabla f( {x}_t^{s+1}) ,  x^* - x^{s+1}_{t+1}  \right \rangle
		\\  & =  & \nonumber   \mathbb{E} \left \langle  {u}_{t}^{s+1} - \nabla f( {x}_t^{s+1}) ,  x^* - x^{s}_{t+1}  \right \rangle +  \mathbb{E} \left \langle   {v}_{t}^{s+1}  - {u}_{t}^{s+1} ,  x^* - x^{s+1}_{t+1}  \right \rangle
		\\  & =  & \nonumber   \underbrace{\mathbb{E} \left \langle  {u}_{t}^{s+1} - \nabla f( {x}_t^{s+1}) ,  x^* - \bar{x}^{s+1}_{t+1}  \right \rangle}_{=0} + \mathbb{E} \left \langle  {u}_{t}^{s+1} - \nabla f( {x}_t^{s+1}) ,  \bar{x}_{t+1}^{s+1} - x^{s+1}_{t+1}  \right \rangle \\
		&& \nonumber +  \mathbb{E} \left \langle   {v}_{t}^{s+1}  - {u}_{t}^{s+1} ,  x^* - x^{s+1}_{t+1}  \right \rangle
		\\  & \leq  & \nonumber   \mathbb{E} \left \|  {u}_{t}^{s+1} - \nabla f( {x}_t^{s+1}) \right \| \left  \|  \bar{x}_{t+1}^{s+1} - x^{s+1}_{t+1} \right  \| +  \mathbb{E} \left \langle   {v}_{t}^{s+1}  - {u}_{t}^{s+1} ,  x^* - x^{s+1}_{t+1}  \right \rangle
		\\  & \leq  & \nonumber   \eta \mathbb{E} \left \|  {u}_{t}^{s+1} - \nabla f( {x}_t^{s+1}) \right \| \left  \|  \nabla f( {x}_t^{s+1}) - {v}_{t}^{s+1} \right  \| +  \mathbb{E} \left \langle   {v}_{t}^{s+1}  - {u}_{t}^{s+1} ,  x^* - x^{s+1}_{t+1}  \right \rangle
		\\  & \leq  & \nonumber   \eta \mathbb{E} \left \|  {u}_{t}^{s+1} - \nabla f( {x}_t^{s+1}) \right \| \left ( \left  \|  \nabla f( {x}_t^{s+1}) - {u}_{t}^{s+1} \right  \| + \left  \|  {u}_{t}^{s+1}  - {v}_{t}^{s+1} \right  \|  \right )  \\
		&& \nonumber + \frac{1}{2\alpha} \mathbb{E} \left \|   {v}_{t}^{s+1}  - {u}_{t}^{s+1} \right \|^2 +   \frac{\alpha}{2} \mathbb{E}\left \|  x^* - x^{s+1}_{t+1}  \right \|^2
		\\  & \leq  & \nonumber   \eta \mathbb{E} \left \|  {u}_{t}^{s+1} - \nabla f( {x}_t^{s+1}) \right \|^2 + \frac{\eta}{2} \mathbb{E} \left \| {u}_{t}^{s+1} - \nabla f( {x}_t^{s+1}) \right \|^2 + \frac{\eta}{2} \left  \|  {u}_{t}^{s+1}  - {v}_{t}^{s+1} \right  \|^2 \\
		&& \nonumber + \frac{1}{2\alpha} \mathbb{E} \left \|   {v}_{t}^{s+1}  - {u}_{t}^{s+1} \right \|^2 +   \frac{\alpha}{2} \mathbb{E}\left \|  x^* - x^{s+1}_{t+1}  \right \|^2
		\\  & =  &    \frac{3}{2}\eta  {  \mathbb{E} \left \|  {u}_{t}^{s+1} - \nabla f( {x}_t^{s+1}) \right \|^2} + \left ( \frac{\eta}{2} +  \frac{1}{2\alpha} \right ) {\mathbb{E}  \left  \|  {u}_{t}^{s+1}  - {v}_{t}^{s+1} \right  \|^2} +   \frac{\alpha}{2} \mathbb{E}\left \|  x^* - x^{s+1}_{t+1}  \right \|^2 \nonumber \\
		& \leq  & \nonumber  \left ( 6 \eta L_f + \left ( \frac{\eta}{2} +  \frac{1}{2\alpha} \right ) \frac{32 }{\mu}\left ( \frac{B_F^2L_G^2}{B} + \frac{B_G^4L_F^2}{A} \right )  \right ) \mathbb{E} \left ( H({x}_t^{s+1}) -H({x}^*) + H(\tilde{x}^{s}) -H({x}^*) \right ) \nonumber \\
		&   &   + \frac{\alpha }{\mu} \mathbb{E} \left ( H({x}_{t+1}^{s+1}) -H({x}^*)\right )
	\end{eqnarray}
	where the third inequality follows from triangle inequality and Young's inequality, the last inequality follows from Lemma \ref{lem1001}.
	This completes the proof.\\ \\
\end{proof}

\noindent \textbf{Proof to Theorem \ref{AsySPSAGA_theorem1}}
\begin{proof}
	We define the virtual  gradient $g_t^{s+1}$ as follows:
	\begin{eqnarray}
		g_t^{s+1} =  \frac{1}{\eta}\left ( x^{s+1}_t - \textrm{Prox}_{ \eta h(\cdot) }\left ( x^{s+1}_t - {\eta} {v}^{s+1}_{t}   \right ) \right )
	\end{eqnarray}
	Because $x^{s+1}_{t+1} = {\arg \min}_{x} \frac{1}{2 \eta } \|x- ( x_{t}^{s+1} - {\eta} {v}^{s+1}_t    )  \|^2 + h(x) $, based on the optimality condition,  we have that $ \frac{1}{\eta}\left ( -x_{t+1}^{s+1}+ ( x_{t}^{s+1} - {\eta} {v}_{t}^{s+1}    ) \right ) \in\partial h(x_{t+1}^{s+1}) $. Thus, we have:
	\begin{eqnarray} \label{Virtual_gradient_equality}
		g_t^{s+1} = \frac{1}{\eta}\left ( x_{t}^{s+1} -x_{t+1}^{s+1} \right ) = {v}^{s+1}_{t} + \xi_{t}^{s+1}
	\end{eqnarray}
	where $\xi_{t}^{s+1} \in \partial h(x_{t+1}^{s+1})$.   Note that the virtual  gradients $g_t^{s+1}$ are only used for the analysis, not computed in the implementation.  Because $x^*$ is the optimal solution,  we have:
	\begin{eqnarray}\label{AsySPSAGA_theorem1_1_1}
		\nonumber H(x^*)& =& f(x^*) + h(x^*)
		\\  & \geq & \nonumber f(x^{s+1}_{t}) + \langle \nabla f(x_t^{s+1}),x^* - x_t^{s+1} \rangle + h(x_{t+1}^{s+1}) + \langle \xi_{t}^{s+1}, x^* - x_{t+1}^{s+1} \rangle
		\\  & \geq & \nonumber f({x}_{t+1}^{s+1}) - \langle \nabla f(x_t^{s+1}),x_{t+1}^{s+1} - x_t^{s+1} \rangle -  \frac{L_f}{2} \| x^{s+1}_{t+1} - x^{s+1}_t \|^2
		\\  &  & \nonumber  + \langle \nabla f(x_t^{s+1}),x^* - x_t^{s+1} \rangle + h(x_{t+1}^{s+1}) + \langle \xi_{t}^{s+1},x^* - x_{t+1}^{s+1} \rangle
		\\  & = & \nonumber H({x}_{t+1}^{s+1}) +  \langle  v_{t}^{s+1}+ \xi_{t}^{s+1}, x^{*} - {x}_t^{s+1} \rangle + \langle  v_{t}^{s+1}+ \xi_{t}^{s+1}, {x}_t^{s+1} - {x}_{t+1}^{s+1} \rangle
		\\  &  & \nonumber  +  \langle \nabla f(x_t^{s+1}) - v_{t}^{s+1}, x^* - x_{t+1}^{s+1} \rangle -  \frac{L_f}{2} \| x_{t+1}^{s+1} - x_t^{s+1} \|^2
		\\  & = &  H({x}_{t+1}^{s+1}) +  \langle  g_t^{s+1},x^{*} - {x}_t^{s+1} \rangle +\eta  \|g_t^{s+1}\|^2
		+ \langle \nabla f(x_t^{s+1}) - v_{t}^{s+1},x^* - x_{t+1}^{s+1} \rangle
		-  \frac{L_f\eta^2}{2}\left \|  g_t^{s+1}  \right \|^2
	\end{eqnarray}
	where the first inequality uses the convexity of $F(x)$ and $R(x)$, the second inequality follows from that $f(x)$ is $L_f$-smooth. Thus, we have:
	\begin{eqnarray}\label{AsySPSAGA_theorem1_1_2}
		&&  \langle  g_t^{s+1},x^{*} - {x}_t^{s+1} \rangle
		\nonumber	\\  & \leq &   H(x^*) - H({x}_{t+1}^{s+1})  - \eta  \|g_t^{s+1}\|^2
		- \langle \nabla f(x_t^{s+1}) - v_{t}^{s+1},x^* - x_{t+1}^{s+1} \rangle
		+  \frac{L_f\eta^2}{2}\left \|  g_t^{s+1}  \right \|^2
	\end{eqnarray}
	
	As per the iteration of $x_t^{s+1}$, we have:
	\begin{eqnarray}\label{AsySPSAGA_theorem1_1}
		&& \mathbb{E}\left \| x^{s+1}_{t+1} -x^*  \right \|^2 = \mathbb{E}\left \| x^{s+1}_{t} - \eta g_t^{s+1} -x^*  \right \|^2
		\nonumber	\\  & = & \nonumber \mathbb{E}\left \| x_t^{s+1}    - x^* \right \|^2 + \eta^2 \mathbb{E}\left \|  g_t^{s+1}   \right \|^2 - 2 \eta \mathbb{E}\left \langle  g_t^{s+1} ,x_t^{s+1}    - x^* \right \rangle
		\\  & \leq & \nonumber \mathbb{E}\left \| x_t^{s+1}    - x^* \right \|^2 + \eta^2 \mathbb{E}\left \|  g_t^{s+1}   \right \|^2  - 2 \eta ( \mathbb{E}H(x_{t+1}^{s+1}) -H(x^*) )+
		\\  &  & \nonumber      2 \eta \mathbb{E} \langle \nabla f( {x}_t^{s+1}) - v_{t}^{s+1}, x_{t+1}^{s+1} - x^* \rangle  - 2\eta^2 \mathbb{E}\left \|  g_t^{s+1}  \right \|^2 +  {L_f\eta^3} \left \|  g_t^{s+1}  \right \|^2
		\\  & {\leq}  & \nonumber \mathbb{E}\left \| x_t^{s+1}    - x^* \right \|^2  - 2 \eta ( \mathbb{E}H(x_{t+1}^{s+1}) -H(x^*) ) +
		2 \eta \mathbb{E} \langle \nabla f( {x}_t^{s+1}) - v_{t}^{s+1}, x_{t+1}^{s+1} - x^* \rangle
		\\  & \leq & \nonumber \mathbb{E}\left \| x_t^{s+1}    - x^* \right \|^2  - 2 \eta ( \mathbb{E}H(x_{t+1}^{s+1}) -H(x^*) ) + 2 \eta \frac{\alpha}{\mu} \mathbb{E} \left ( H({x}_{t+1}^{s+1}) -H({x}^*)\right )
		\\  &  & \nonumber  + 2 \eta  \left ( 6 \eta L_f + \left ( \frac{\eta}{2} +  \frac{1}{2\alpha} \right ) \frac{32}{\mu}\left ( \frac{B_F^2L_G^2}{B} + \frac{B_G^4L_F^2}{A} \right )   \right ) \mathbb{E} \left ( H({x}_t^{s+1}) -H({x}^*) + H(\tilde{x}^{s}) -H({x}^*) \right )
		\\  & = &  \mathbb{E}\left \| x_t^{s+1}    - x^* \right \|^2  - 2 \eta \left (1-\frac{\alpha }{\mu} \right ) ( \mathbb{E}H(x_{t+1}^{s+1}) -H(x^*) )
		\\  &  & \nonumber  + 2 \eta  \left ( 6 \eta L_f + \left ( \frac{\eta}{2} +  \frac{1}{2\alpha} \right ) \frac{32 }{\mu}\left ( \frac{B_F^2L_G^2}{B} + \frac{B_G^4L_F^2}{A} \right )  \right ) \mathbb{E} \left ( H({x}_t^{s+1}) -H({x}^*) + H(\tilde{x}^{s}) -H({x}^*) \right )
	\end{eqnarray}
	where the first inequality follows from (\ref{AsySPSAGA_theorem1_1_2}), the second inequality follows from $\eta L_f \leq 1$, the third inequality follows from Lemma \ref{SVRG-FSC-2_lem2}.  By summing
	the  inequality (\ref{AsySPSAGA_theorem1_1}) over $t = 0,\cdots,m-1$, because $\tilde{x}^{s+1} = x_m^{s+1}$ and $\tilde{x}^s = x_0^{s+1}$, we obtain the following inequality:
	\begin{eqnarray}\label{AsySPSAGA_theorem1_2}
		&&  \nonumber \mathbb{E}\left \| x^{s+1}_{m} -x^*  \right \|^2 + 2 \eta \left (1-\frac{\alpha }{\mu} \right ) ( \mathbb{E}H(x_{m}^{s+1}) -H(x^*) ) + 2 \eta \left (1-\frac{\alpha }{\mu} \right ) \sum_{t=1}^{m-1} ( \mathbb{E}H(x_{t}^{s+1}) -H(x^*) )
		\\  & \leq & \nonumber \mathbb{E}\left \| \tilde x^{s} -x^*  \right \|^2 + 2 \eta  \left ( 6 \eta L_f + \left ( \frac{\eta}{2} +  \frac{1}{2\alpha} \right )\frac{32 }{\mu}\left ( \frac{B_F^2L_G^2}{B} + \frac{B_G^4L_F^2}{A} \right )   \right ) \sum_{t=1}^{m-1} \mathbb{E} \left ( H({x}_t^{s+1}) -H({x}^*) \right )
		\\  &  &  +  2 \eta  \left ( 6 \eta L_f + \left ( \frac{\eta}{2} +  \frac{1}{2\alpha} \right ) \frac{32 }{\mu} \left ( \frac{B_F^2L_G^2}{B} + \frac{B_G^4L_F^2}{A} \right )  \right ) (m+1) \mathbb{E} \left ( H(\tilde {x}^{s}) -H({x}^*) \right ) 
	\end{eqnarray}
	According to (\ref{AsySPSAGA_theorem1_2}) and $ \left \| \tilde x^{s} -x^*  \right \|^2 \leq \frac{2}{\mu} (H(\tilde{x})- H(x^*) )$ , the following inequality holds:
	\begin{eqnarray}\label{AsySPSAGA_theorem1_4}
		&  & \nonumber 2 \eta \left (1-\frac{\alpha }{\mu} -\left ( 6 \eta L_f + \left ( \frac{\eta}{2} +  \frac{1}{2\alpha} \right ) \frac{32 }{\mu}\left ( \frac{B_F^2L_G^2}{B} + \frac{B_G^4L_F^2}{A} \right )   \right ) \right ) \sum_{t=1}^{m} ( \mathbb{E}H(x_{t}^{s+1}) -H(x^*) )
		\\  & \leq &   \left ( \frac{2 }{\mu} + 2 \eta  \left ( 6 \eta L_f + \left ( \frac{\eta}{2} +  \frac{1}{2\alpha} \right ) \frac{32 }{\mu}\left ( \frac{B_F^2L_G^2}{B} + \frac{B_G^4L_F^2}{A} \right )   \right ) (m+1) \right ) \mathbb{E} \left ( H(\tilde{x}^{s}) -H({x}^*) \right )
	\end{eqnarray}
	As per the convexity of $H(x)$ and the definition of $\tilde{x}^{s+1}$, we have $H(\tilde{x}^{s+1}) \leq \frac{1}{m} \sum\limits_{t=1}^m H(x_{t}^{s+1}) $. According to (\ref{AsySPSAGA_theorem1_4}), and let $\alpha=\frac{\mu }{ 8 }$, we have that:
	\begin{eqnarray}\label{AsySPSAGA_theorem1_5}
		&  & \nonumber   \mathbb{E}H(\tilde{x}^{s+1}) -H(x^*) 
		\\  & \leq &  \frac{  \frac{2 }{\mu} + 2 \eta  \left ( 6 \eta L_f + \left ( \frac{\eta}{2} +  \frac{4 }{\mu} \right ) \frac{32 }{\mu}\left ( \frac{B_F^2L_G^2}{B} + \frac{B_G^4L_F^2}{A} \right )  \right ) (m+1) }{2 \eta \left (\frac{7}{8} -\left ( 6 \eta L_f + \left ( \frac{\eta}{2} +  \frac{4 }{\mu} \right ) \frac{32 }{\mu}\left ( \frac{B_F^2L_G^2}{B} + \frac{B_G^4L_F^2}{A} \right ) \right ) \right )m} \mathbb{E}  H(\tilde{x}^{s}) -H({x}^*) 
	\end{eqnarray}
	Define $\rho = \frac{  \frac{2 }{\mu} + 2 \eta  \left ( 6 \eta L_f + \left ( \frac{\eta}{2} +  \frac{4 }{\mu} \right ) \frac{32 }{\mu}\left ( \frac{B_F^2L_G^2}{B} + \frac{B_G^4L_F^2}{A} \right )  \right ) (m+1) }{2 \eta \left (\frac{7}{8} -\left ( 6 \eta L_f + \left ( \frac{\eta}{2} +  \frac{4 }{\mu} \right ) \frac{32 }{\mu}\left ( \frac{B_F^2L_G^2}{B} + \frac{B_G^4L_F^2}{A} \right ) \right ) \right )m} $, where $\rho < 1$.  Applying (\ref{AsySPSAGA_theorem1_5}) recursively, we have:
	\begin{eqnarray}
		\mathbb{E} H(\tilde{x}^{S}) - H(x^*) \leq \rho^S \bigg(H(x^0) - H(x^*)\biggr)
	\end{eqnarray}
	This completes the proof. \\\\	
\end{proof}

\noindent \textbf{Proof to Corollary \ref{AsySPSAGA_theorem2}}
\begin{proof}
	Appropriately choosing $\eta$, $m$ and $A$ in Algorithm \ref{algorithmSVRG2},  we can get the linear convergence rate for our algorithm. We choose $A$ and $B$ to  satisfy:
	\begin{eqnarray}
		\frac{\eta}{2}\frac{32 B_F^2L_G^2}{B\mu } & \leq& \frac{1}{16}
		\\ \frac{\eta}{2}\frac{32 B_G^4L_F^2}{A\mu } & \leq& \frac{1}{16}
		\\ \frac{4 }{\mu}\frac{32 B_F^2L_G^2}{B\mu }  &\leq& \frac{1}{16}
		\\ \frac{4}{\mu}\frac{32 B_G^4L_F^2}{A\mu }  &\leq& \frac{1}{16}
	\end{eqnarray}
	The above inequalities are equivalent with the following conditions:
	\begin{eqnarray}
		A  & \geq & \max \left \{ \frac{256 \eta  B_G^4L_F^2}{\mu } , \frac{2048  B_G^4L_F^2}{\mu^2 }  \right \}
		\\ B  &\geq & \max \left \{ \frac{256 \eta B_F^2L_G^2}{\mu } , \frac{2048  B_F^2L_G^2}{ \mu^2 }  \right \}
	\end{eqnarray}
	We choose $\eta$ satisfying $ 6 \eta L_f \leq \frac{1}{16}$, such that $\eta \leq \frac{1}{96 L_f}$. It follows that:
	\begin{eqnarray}\label{AsySPSAGA_corollary1}
		&& \nonumber \frac{  \frac{2 }{\mu} + 2 \eta  \left ( 6 \eta L_f + \left ( \frac{\eta}{2} +  \frac{4 }{\mu} \right ) \frac{32 }{\mu}\left ( \frac{B_F^2L_G^2}{B} + \frac{B_G^4L_F^2}{A} \right )  \right ) (m+1) }{2 \eta \left (\frac{7}{8} -\left ( 6 \eta L_f + \left ( \frac{\eta}{2} +  \frac{4 }{\mu} \right ) \frac{32 }{\mu}\left ( \frac{B_F^2L_G^2}{B} + \frac{B_G^4L_F^2}{A} \right ) \right ) \right )m}
		\\  & \leq & \nonumber \frac{  \frac{2 }{\mu m} + 2 \eta  \frac{5}{16} (1+\frac{1}{m}) }{2 \eta \frac{9}{16} }
		\\  & = &  \frac{5}{9} + \frac{  \frac{2 }{\mu m} + 2 \eta  \frac{1}{m} }{2 \eta \frac{9}{16} }
	\end{eqnarray}
	We then set $m$ satisfying
	$
	\frac{  \frac{2 }{\mu m} + 2 \eta  \frac{1}{m} }{2 \eta \frac{9}{16}  } \leq \frac{1}{9}
	$, which is equivalent to $m \geq  16 \left ( 1 +  \frac{ 1}{\mu \eta} \right )   $.
	Thus, by setting $\eta$, $m$, $A$ and $B$ appropriately:
	\begin{eqnarray}
		\eta &=& \frac{1}{96 L_f}
		\\   m &=&  16 \left ( 1 +  \frac{ 96 L_f}{\mu } \right )
		\\   A&=& \max \left \{ \frac{8  B_G^4L_F^2}{3 \mu L_f } , \frac{2048  B_G^4L_F^2}{\mu^2 }  \right \}
		\\   B&=& \max \left \{ \frac{8  B_F^2L_G^2}{3 \mu L_f } , \frac{2048  B_F^2L_G^2}{ \mu^2 }  \right \}
	\end{eqnarray}
	we can obtain a linear convergence rate $\frac{2}{3}$.
	This completes the proof.
	
\end{proof}

\section{General Problem }
\noindent \textbf{Proof to Theorem \ref{thm3}}
\begin{proof}
	We define $u_t^{s+1}$ to be an unbiased estimate for $\nabla f(x_t^{s+1}) $:
	\begin{eqnarray}
		{u}_{t}^{s+1} &=& \frac{1}{b_1} \sum\limits_{i_t \in I_t}  \biggl( \left ( \nabla  G({x}_{t}^{s+1}) \right )^T \nabla  F_{i_t}( {G}(x_t^{s+1}))  -\left ( \nabla  G( \tilde{x}^{s}) \right )^T \nabla  F_{i_t}({G}^{s}) \biggr) + \nabla  f(\tilde{x}^{s})
	\end{eqnarray}
	where $I_t$ is uniformly sampled from $\{1,2,...,n_1 \}$ such that $|I_t|=b_1$. We also define:
	\begin{eqnarray}
		\bar{x}_{t+1}^{s+1}& =& \text{Prox}_{\eta h(.)} (x_{t+1}^{s+1} - \eta \nabla f(x_t^{s+1}))
	\end{eqnarray}
	According to Lemma 2 in \cite{reddi2016proximal}, if $y=prox_{\eta h}(x - \eta d'), d' \in \mathbb{R}^d$ and $\forall z \in \mathbb{R}^d$, following equality holds that:
	\begin{eqnarray}
		H(y) \leq H(z) + \left< y-z, \nabla f(x) - d' \right> + \bigg[\frac{L_f}{2} - \frac{1}{2\eta} \bigg] \|y-x\|^2 + \bigg[\frac{L_f}{2} + \frac{1}{2\eta}\bigg] \|z-x\|^2 - \frac{1}{2\eta} \|y-z\|^2
	\end{eqnarray} 
	Thus if we let $y = \bar{x}_{t+1}^{s+1}$, $x = z=x_t^{s+1}$ and $d' = \nabla f(x_{t}^{s+1}) $, we have:
	\begin{eqnarray}
		\mathbb{E}[H(\bar{x}_{t+1}^{s+1})] &\leq& \mathbb{E} \bigg[ H(x_t^{s+1}) + \bigg[\frac{L_f}{2} - \frac{1}{2\eta}\bigg] \|\bar{x}_{t+1}^{s+1} - x_t^{s+1} \|^2  - \frac{1}{2\eta} \|\bar{x}_{t+1}^{s+1} -  x_t^{s+1} \|^2 \bigg].
		\label{iq7001}
	\end{eqnarray}
	If we let $y=x_{t+1}^{s+1}$, $x= x_{t}^{s+1}$,  $z = \bar{x}_{t+1}^{s+1}$ and $d'=v_t^{s+1}$, then we have:
	\begin{eqnarray}
		\label{iq7002}
		\mathbb{E}[H({x}_{t+1}^{s+1})] &\leq& \mathbb{E} \bigg[ H(x_t^{s+1}) + \bigg[\frac{L_f}{2} - \frac{1}{2\eta}\bigg] \|\bar{x}_{t+1}^{s+1} - x_t^{s+1} \|^2 
		+ \bigg[ \frac{L_f}{2} + \frac{1}{2\eta} \bigg]  \|x_{t+1}^{s+1} - x_t^{s+1}  \|^2  \nonumber \\	
		&&	  - \frac{1}{2\eta} \|\bar{x}_{t+1}^{s+1} -  x_t^{s+1} \|^2 + \left< x_{t+1}^{s+1} - \bar{x}_{t+1}^{s+1}, \nabla f(x_t^{s+1}) - v_t^{s+1}  \right>\bigg] .
	\end{eqnarray}
	
	Combine (\ref{iq7001}) and (\ref{iq7002}), we have: 
	\begin{eqnarray}
		\label{iq8001}
		\mathbb{E} [H(x_{t+1}^{s+1} )  ] & \leq & \mathbb{E} \bigg[ H(x_t^{s+1})  + \bigg[L_f - \frac{1}{2\eta}\bigg] \| \bar{x}_{t+1}^{s+1} - x_t^{s+1} \|^2  + \bigg[ \frac{L_f}{2} - \frac{1}{2\eta}  \bigg] \| x_{t+1}^{s+1} - x_t^{s+1} \|^2  \nonumber \\
		&& - \frac{1}{2\eta} \| x_{t+1}^{s+1} - \bar{x}_{t+1}^{s+1} \|^2  +  \underbrace{\left< x_{t+1}^{s+1} - \bar{x}_{t+1}^{s+1}, \nabla f(x_t^{s+1}) - v_t^{s+1}  \right>}_{T_1} \bigg].
	\end{eqnarray}
	
	Then we can have an upper bound of $T_1$: 
	\begin{eqnarray}
		\mathbb{E} [T_1]  &\leq& \frac{1}{2\eta } \mathbb{E} \| x_{t+1}^{s+1} - \tilde{x}_{t+1}^{s+1} \|^2 + \frac{\eta}{2} \| \nabla f(x_t^{s+1}) - v_t^{s+1} \|^2 \nonumber \\
		&\leq & \frac{1}{2\eta } \mathbb{E} \|x_{t+1}^{s+1} - \tilde{x}_{t+1}^{s+1} \|^2 +  {\eta} \mathbb{E}  \| \nabla f(x_t^{s+1}) - u_t^{s+1} \|^2  +  {\eta} \mathbb{E} \| v_t^{s+1} - u_t^{s+1} \|^2 
	\end{eqnarray}
	where the first inequality follows from Cauchy-Schwarz and Young's inequality. As per the definition of $u_t^{s+1}$, we have:
	\begin{eqnarray}
		\label{iq9001}
		&& \mathbb{E}  \| \nabla f(x_t^{s+1}) - u_t^{s+1} \|^2 \nonumber \\
		&=&    \mathbb{E}  \| \frac{1}{b_1} \sum\limits_{i_t \in I_t}   \biggl( \nabla f(x_t^{s+1}) - \left ( \nabla  G({x}_{t}^{s+1}) \right )^T \nabla  F_{i_t}( {G}(x_t^{s+1}))  \nonumber \\
		&& + \left ( \nabla  G( \tilde{x}^{s}) \right )^T \nabla  F_{i_t}({G}^{s}) - \nabla  f(\tilde{x}^{s}) \biggr)\|^2 \nonumber \\
		&=&     \frac{1}{b_1^2} \mathbb{E}  \| \sum\limits_{i_t \in I_t}  \biggl( \nabla f(x_t^{s+1}) - \left ( \nabla  G({x}_{t}^{s+1}) \right )^T \nabla  F_{i_t}( {G}(x_t^{s+1}))  \nonumber \\
		&& + \left ( \nabla  G( \tilde{x}^{s}) \right )^T \nabla  F_{i_t}({G}^{s}) - \nabla  f(\tilde{x}^{s}) \biggr)\|^2 \nonumber \\
		&=&     \frac{1}{b_1^2} \mathbb{E}  \sum\limits_{i_t \in I_t}    \| \nabla f(x_t^{s+1}) - \left ( \nabla  G({x}_{t}^{s+1}) \right )^T \nabla  F_{i_t}( {G}(x_t^{s+1}))  \nonumber \\
		&& + \left ( \nabla  G( \tilde{x}^{s}) \right )^T \nabla  F_{i_t}({G}^{s}) - \nabla  f(\tilde{x}^{s}) \|^2 \nonumber \\
		&\leq &  \frac{1}{b_1^2} \mathbb{E}  \sum\limits_{i_t \in I_t}  \| \left ( \nabla  G({x}_{t}^{s+1}) \right )^T \nabla  F_{i_t}( {G}(x_t^{s+1}))  - \left ( \nabla  G( \tilde{x}^{s}) \right )^T \nabla  F_{i_t}({G}^{s}) \|^2 \nonumber \\
		&\leq & \frac{L_f^2}{b_1} \mathbb{E}\| x_t^{s+1} - \tilde{x}^s \|^2
	\end{eqnarray}
	where  the third equality follows from Lemma \ref{ex_lem1}, the first inequality follows from:
	\begin{eqnarray}
		\label{variance}
		\mathbb{E} \|\xi - \mathbb{E}[\xi] \|^2& =& \mathbb{E} \| \xi \|^2 - \| \mathbb{E}[\xi]\|^2, 
	\end{eqnarray} and the last inequality follows from Assumption \ref{definition1}. We can  bound $\mathbb{E} \| v_t^{s+1} - u_t^{s+1} \|^2 $ as follows:
	\begin{eqnarray}
		\label{iq9002}
		&& \mathbb{E} \| v_t^{s+1} - u_t^{s+1} \|^2  \nonumber \\
		&= & \mathbb{E} \|  \frac{1}{b_1} \sum\limits_{i_t \in I_t}   \biggl( \left ( \nabla  \widehat G_{t}^{s+1} \right )^T \nabla  F_{i_t}( \widehat{G}_t^{s+1}) - \left ( \nabla  G({x}_{t}^{s+1}) \right )^T \nabla  F_{i_t}( G(x_t^{s+1}))   \biggr)  \|^2 \nonumber \\
		&\leq & \frac{1}{b_1} \mathbb{E} \sum\limits_{i_t \in I_t}   \|     \left ( \nabla  \widehat G_{t}^{s+1} \right )^T \nabla  F_{i_t}( \widehat{G}_t^{s+1}) - \left ( \nabla   G({x}_{t}^{s+1}) \right )^T \nabla  F_{i_t}( G(x_t^{s+1}))    \|^2 \nonumber \\			
		&\leq&  \frac{1}{b_1} \sum\limits_{i_t \in I_t}  \bigg[ \frac{2B_G^2 L_F^2}{A^2} \sum\limits_{1\leq j \leq A }  \mathbb{E} \| ( G_{\mathcal{A}_k[j]} (\tilde{x}^s) - G_{\mathcal{A}_k[j]} (x_t^{s+1}) ) - (G^s - G(x_t^{s+1}))   \|^2  \nonumber \\
		& & + \frac{2B_F^2 }{B^2} \sum\limits_{1 \leq j \leq B} \| \nabla G_{\mathcal{B}_t[j]}(x_t^{s+1})  - \nabla G_{\mathcal{B}_t[j]}(\tilde x) + \nabla G(\tilde{x}) - \nabla G(x_{t}^{s+1}) \|^2 \biggr] \nonumber \\
		&\leq&  \frac{1}{b_1} \sum\limits_{i_t \in I_t}  \bigg[ \frac{2B_G^2 L_F^2}{A^2} \sum\limits_{1\leq j \leq A }  \mathbb{E} \| G_{\mathcal{A}_k[j]} (\tilde{x}^s) - G_{\mathcal{A}_k[j]} (x_t^{s+1})    \|^2  \nonumber \\
		& & + \frac{2B_F^2 }{B^2} \sum\limits_{1 \leq j \leq B}\mathbb{E} \| \nabla G_{\mathcal{B}_t[j]}(x_t^{s+1})  - \nabla G_{\mathcal{B}_t[j]}(\tilde x) \|^2 \biggr] \nonumber\\
		&\leq & ( \frac{2B_G^4 L_F^2}{A} 	+  \frac{2B_F^2 L_G^2 }{B}   ) \mathbb{E} \| x_t^{s+1} - \tilde{x}^s \|^2 	
	\end{eqnarray} 
	where the first inequality follows from Lemma \ref{ex_lem2}, the second inequality follows from (\ref{fuck}), the third inequality follows from (\ref{variance}) and the last inequality follows from the Assumption of bounded gradients.
	From (\ref{iq9001}) and (\ref{iq9002}), we get the upper bound of $T_1$:
	\begin{eqnarray}
		\mathbb{E} [T_1] &\leq& \frac{1}{2\eta} \mathbb{E} \|x_{t+1}^{s+1} - \tilde{x}_{t+1}^{s+1} \|^2 +  \frac{\eta L_f^2}{b_1} \mathbb{E}\| x_t^{s+1} - \tilde{x}^s \|^2 + 
		( \frac{2B_G^4 L_F^2}{A} 	+  \frac{2B_F^2 L_G^2 }{B}   )  \mathbb{E} \| x_t^{s+1} - \tilde{x}^s \|^2.
	\end{eqnarray}
	Substituting the upper bound of $T_1$ in (\ref{iq8001}), we have:
	\begin{eqnarray}
		\mathbb{E} [H(x_{t+1}^{s+1} )  ] & \leq & \mathbb{E} \bigg[ H(x_t^{s+1})  + \bigg[L_f - \frac{1}{2\eta}\bigg] \| \bar{x}_{t+1}^{s+1} - x_t^{s+1} \|^2  + \bigg[ \frac{L_f}{2} - \frac{1}{2\eta}  \bigg] \| x_{t+1}^{s+1} - x_t^{s+1} \|^2  \nonumber \\
		&& + \biggl( \frac{\eta L_f^2}{b_1}+ 
		\frac{2B_G^4 L_F^2}{A} 	+  \frac{2B_F^2 L_G^2 }{B}   \biggr) \mathbb{E}\| x_t^{s+1} - \tilde{x}^s \|^2  \bigg].
	\end{eqnarray}
	Following the analysis in \cite{reddi2016fast}, we define the Lyapunov function
	$
	R_t^{s+1} = \mathbb{E} [ H(x_t^{s+1}) + c_t \| x_t^{s+1} - \tilde{x}^s \|^2 ]
	$. Then we have the upper bound of $R_{t+1}^{s+1}$:
	\begin{eqnarray}
		\label{iq_final_1}
		R_{t+1}^{s+1} &=& 	\mathbb{E} \biggl[ H(x_{t+1}^{s+1}) + c_{t+1} \| x_{t+1}^{s+1}  - x_t^{s+1} + x_t^{s+1} - \tilde{x}^s \|^2 ]  \nonumber \\
		&\leq & 	\mathbb{E} [ H(x_{t+1}^{s+1}) + c_{t+1} (1+\frac{1}{\beta}) \| x_{t+1}^{s+1}  - x_t^{s+1} \|^2 + c_{t+1} (1+{\beta}) \| x_t^{s+1} - \tilde{x}^s \|^2 \biggr]  \nonumber \\
		&\leq & 	\mathbb{E} \biggl[ H(x_{t}^{s+1})   + \bigg[L_f - \frac{1}{2\eta}\bigg] \| \bar{x}_{t+1}^{s+1} - x_t^{s+1} \|^2  + \bigg[  c_{t+1} (1+ \frac{1}{\beta}) + \frac{L_f}{2} - \frac{1}{2\eta}  \bigg] \| x_{t+1}^{s+1} - x_t^{s+1} \|^2  \nonumber \\
		&& +  \biggl( c_{t+1} (1+{\beta}) + \frac{\eta L_f^2}{b_1} + 
		\frac{2\eta B_G^4 L_F^2}{A} 	+  \frac{2 \eta B_F^2 L_G^2 }{B}     \biggr) \mathbb{E}\| x_t^{s+1} - \tilde{x}^s \|^2
		\biggr]  \nonumber \\
		&\leq& R_t^{s+1} +  \bigg[L_f - \frac{1}{2\eta}\bigg] \mathbb{E} \| \bar{x}_{t+1}^{s+1} - x_t^{s+1} \|^2 
	\end{eqnarray}
	where the last inequality follows from that: 
	\begin{eqnarray}
		&&c_t =  c_{t+1} (1+ \frac{1}{\beta}) + \frac{\eta L_f^2}{b_1} +  \frac{2 \eta B_G^4 L_F^2}{A} 	+  \frac{2\eta B_F^2 L_G^2 }{B} \\
		&&c_{t+1} (1+ \frac{1}{\beta}) + \frac{L_f}{2} \leq  \frac{1}{2\eta} 
	\end{eqnarray}
	It is easy to know that $c_t$ is decreasing, if we let $c_m=0$ and $\beta = \frac{1}{m}$, then we have:
	\begin{eqnarray}
		c_0 &=& \biggl( \frac{\eta m L_f^2}{b_1}+   \frac{2\eta m B_G^4 L_F^2}{A} 	+  \frac{2 \eta m B_F^2 L_G^2 }{B}   \biggr) ( (1+\frac{1}{m})^m - 1 ) \nonumber \\
		&\leq & \biggl( \frac{\eta m L_f^2}{b_1}+   \frac{2\eta m B_G^4 L_F^2}{A} 	+  \frac{2 \eta m B_F^2 L_G^2 }{B}   \biggr)  (e-1)
	\end{eqnarray}
	where the inequality follows from that $(1+\frac{1}{l})^l$ is increasing if $l> 0$ and $\lim_{l\rightarrow \infty} (1+\frac{1}{l})^l = e$, where $e$ is Euler's number. Therefore, in order to have (\ref{iq_final_1}), it should be satisfied that:
	\begin{eqnarray}
		c_0(1+ \frac{1}{\beta}) + \frac{L_f}{2} &\leq &  2 \biggl( \frac{\eta m L_f^2}{b_1}+   \frac{2\eta m B_G^4 L_F^2}{A} 	+  \frac{2 \eta m B_F^2 L_G^2 }{B}   \biggr)  (1+m) + \frac{L_f}{2}  \nonumber \\
		&\leq &  4 \biggl( \frac{\eta m^2 L_f^2}{b_1}+   \frac{2\eta m^2 B_G^4 L_F^2}{A} 	+  \frac{2 \eta m^2 B_F^2 L_G^2 }{B}   \biggr)  + \frac{L_f}{2}   \nonumber \\ 
		&	\leq & \frac{1}{2\eta}
	\end{eqnarray}
	where the first and second inequalities follow from that $e-1 < 2$ and $m \geq 1$.  Therefore, the condition can be written as:
	\begin{eqnarray}
		4 \biggl( \frac{\eta m^2 L_f^2}{b_1}+   \frac{2\eta m^2 B_G^4 L_F^2}{A} 	+  \frac{2 \eta m^2 B_F^2 L_G^2 }{B}   \biggr)  + \frac{L_f}{2}   &\leq&  \frac{1}{2\eta}
	\end{eqnarray}
	Summing over (\ref{iq_final_1}) from $t=0$ to $m-1$, we have: 
	\begin{eqnarray}
		\label{iq10001}
		R_m^{s+1} &\leq& R^{s+1}_0 +  \sum\limits_{t=0}^{m-1} \biggl[ L_f - \frac{1}{2\eta} \biggr] \mathbb{E} \| \bar{x}_{t+1}^{s+1} - x_t^{s+1} \|^2
	\end{eqnarray}
	Because $R_m^{s+1} = \mathbb{E}[H(x_m^{s+1})] = \mathbb{E}[H(\tilde{x}^{s+1}) ] $ and $R_0^{s+1} = \mathbb{E} [H(x_0^{s+1})] = \mathbb{E} [H(\tilde{x}^s) ] $. Then adding up (\ref{iq10001}) from $s=0$ to $S-1$, we get:
	\begin{eqnarray}
		\sum\limits_{s=0}^{S-1} \sum\limits_{t=0}^{m-1} \bigg[  \frac{1}{2\eta} - L_f \bigg] \mathbb{E} \| \bar{x}_{t+1}^{s+1} - x_t^{s+1} \|^2 \leq H(\tilde{x}^0) - H(\tilde{x}^{S}) \leq H(\tilde x^0) - H(x^*)
	\end{eqnarray}
	where $x^*$ is the optimal solution. We define  $\mathcal{G}_\eta(x_t^{s+1}) = \frac{1}{\eta} (x_t^{s+1} - \bar{x}_{t+1}^{s+1}   ) $ and $x_a$ is uniformly selected from $\{\{x_t^{s+1}\}^{m-1}_{t=0} \}_{t=0}^{S-1}$, we have:  
	\begin{eqnarray}
		\mathbb{E} \| \mathcal{G}_\eta(x_a) \|^2 &\leq& \frac{2}{(1- 2\eta L_f)\eta}  \frac{H(\tilde x^0) - H(x^*)}{T}
	\end{eqnarray}
\end{proof}

\section{Other Lemmas}

\begin{lemma}[\cite{reddi2016fast}]
	For random variables $z_1, . . . , z_r$ are independent and mean 0, we have:
	\begin{eqnarray}
		\mathbb{E} [ \| z_1 + ... + z_r\|^2 ]  &=& \mathbb{E} [\|z_1\|^2 + ...+ \|z_r\|^2 ]
	\end{eqnarray}
	\label{ex_lem1}
\end{lemma}

\begin{lemma}
	For any $z_1,...,z_r$, it holds that:
	\begin{eqnarray}
		\| z_1 + ... + z_r\|^2 &  \leq  & r (\|z_1\|^2 +...+ \|z_r\|^2 )
	\end{eqnarray}
	\label{ex_lem2}
\end{lemma}

\begin{remark}
	\label{foradmm}
	In this remark, we follow the notations in \cite{svradmm}, please check their paper for details.  As per Theorem 1 in \cite{svradmm}, the overall query complexity for com-SVR-ADMM is $O\biggl( (m + n + \kappa^4)\log (1/\varepsilon) \biggr)$ and its convergence rate in terms of query is $O\biggl( \rho^{ \frac{T}{m + n + \kappa^4} } \biggr)$. 	
	\textit{(\textbf{Note:} Corresponding notations in our paper, $n_1 \leftarrow n$, $n_2 \leftarrow m$, $m \leftarrow K$, $A\leftarrow N$, therefore, their query complexity is represented as $O\biggl( \rho^{ \frac{T}{n_1 + n_2 + \kappa^4} } \biggr)$ in Table \ref{table:methods}.)}
\end{remark}

\begin{proof}
	As per Theorem 1 in \cite{svradmm}, they prove that:
	\begin{eqnarray}
		\gamma_1 \mathbb{E} G(\tilde u^s) &\leq& \gamma_2 G(\tilde{u}^{s-1})
	\end{eqnarray}
	where $\gamma_1$ and $\gamma_2$ is defined as follows:
	\begin{eqnarray}
		\gamma_1 &=& (2\eta - \frac{32\eta^2 C_G^4L_f^2}{\mu_F N} -  \frac{48\eta^2 L_F^2 + 8 \eta D C_G L_f L_G N^{-0.5}}{\mu_F}  ) K 
	\end{eqnarray}
	
	\begin{eqnarray}
		\gamma_2 = (K+1) (  \frac{32\eta^2 C_G^4L_f^2}{\mu_F N} + \frac{48\eta^2 L_F^2 + 8 \eta D C_G L_f L_G N^{-0.5}}{\mu_F}     ) + \frac{2}{\mu_F} + 
		\frac{2\eta \rho \|A^TA\|}{\mu_F} + \frac{2L_F \eta}{\rho \sigma_{\min} (AA^T)}
	\end{eqnarray}
	Their conclusion is true under \textbf{two conditions:} (1) $	\gamma_1 > 0$; (2) $	\frac{\gamma_2}{\gamma_1} <  1$. To make $\gamma_1 > 0 $, following inequality should be satisfied that:
	\begin{eqnarray}
		\eta \biggl( 2 - \frac{32\eta C_G^4L_f^2}{\mu_F N} -  \frac{48\eta L_F^2 }{\mu_F} -  \frac{8  D C_G L_f L_G N^{-0.5}}{\mu_F} \biggr) &>& 0
	\end{eqnarray} 
	Suppose that: 
	\begin{eqnarray}
		\frac{32\eta C_G^4L_f^2}{\mu_F N} &<& \frac{1}{4} \\
		\frac{48\eta L_F^2 }{\mu_F} &<& \frac{1}{4} \\
		\frac{8  D C_G L_f L_G N^{-0.5}}{\mu_F} &<&\frac{1}{4}
	\end{eqnarray}
	we have $\gamma_1 > \frac{5\eta K}{4}$.  Therefore, following inequalities should be hold that: 
	\begin{eqnarray}
		N &>& \max  \{ \frac{128 \eta C_G^4 L_f^2}{\mu_F}, \frac{32^2D^2C_G^2L_f^2L_G^2}{\mu_F^2} \} \\
		\eta & < & \frac{\mu_F}{192L_F^2}
	\end{eqnarray}
	Follow the definition of $\kappa$ in the paper, we have $\kappa = \max\{\frac{L_f}{\mu_F}, \frac{L_G}{\mu_F}, \frac{L_F}{\mu_F}  \}$, thus $N$ should be proportional to $\kappa^2$. Given $\gamma_1 = \frac{5\eta K}{4}$, in order to make $\frac{\gamma_2}{\gamma_1}< 1$, following inequality should be satisfied that:
	\begin{eqnarray}
		\frac{3(K+1)}{5K} + \frac{8}{5\eta \mu_F K} +  
		\frac{8 \rho \|A^TA\|}{5  \mu_F K}  + \frac{8L_F }{5  \rho \sigma_{\min} (AA^T) K} &<& 1
	\end{eqnarray}
	To make the above inequality true, at least, it should holds that:
	\begin{eqnarray}
		\frac{8}{5 \eta \mu_F K} &<& 1 
	\end{eqnarray}
	Therefore, we can also obtain that: 
	\begin{eqnarray}
		\frac{8 \times 192 {L_F}^2}{5  \mu_F^2 } &<& K 
	\end{eqnarray}
	So far, we know that $K$ is also proportional to $\kappa^2$. In \cite{svradmm}, the authors claimed that their overall query complexity is $O\biggl((m+n+KN)\log(1/\varepsilon) \biggr)$. Because $K$ and $N$ are proportional to $\kappa^2$, their query complexity can also be represented as $O\biggl((m+n + \kappa^4) \log(1/\varepsilon) \biggr) $ and convergence rate in terms of query is $O\biggl( \rho^{ \frac{T}{m + n + \kappa^4} } \biggr)$. 	
	
\end{proof}

\end{document}